\documentclass{article}

\usepackage{graphicx} % more modern

\usepackage{amsfonts}       % blackboard math symbols
\usepackage{nicefrac}       % compact symbols for 1/2, etc.
\usepackage{microtype}      % microtypography
\usepackage{caption}
\usepackage{booktabs}

\usepackage{algorithm}
\usepackage{algorithmic}

\usepackage{enumitem}
\usepackage{hyperref}
\usepackage{tcolorbox}

\usepackage{fullpage}
\usepackage{authblk}
\usepackage[round]{natbib} 

\usepackage{Definitions}
\usepackage{color}

\renewcommand{\leq}{\leqslant}
\renewcommand{\le}{\leqslant}

\renewcommand{\ge}{\geqslant}

\title{Kernel Exponential Family Estimation via Doubly Dual Embedding}

\author{
  $^*$Bo Dai$^1$, \thanks{indicates equal contribution. Email: \texttt{bodai@google.com, hanjundai@gatech.edu}}$^*$Hanjun Dai$^2$, Arthur Gretton$^3$, Le Song$^2$, Dale Schuurmans$^1$, Niao He$^4$\\
  $^1$ Google Brain, $^2$Georgia Institute of Technology\\
  $^3$University College London, $^4$ University of Illinois at Urbana Champaign\\
}

\begin{document}
% \nipsfinalcopy is no longer used

\maketitle

\begin{abstract}
  We investigate penalized maximum log-likelihood estimation for exponential family distributions whose natural parameter resides in a reproducing kernel Hilbert space. Key to our approach is a novel technique, \emph{doubly dual embedding}, that avoids computation of the partition function. This technique also allows the development of a flexible sampling strategy that amortizes the cost of Monte-Carlo sampling in the inference stage. The resulting estimator can be easily generalized to kernel conditional exponential families. We establish a connection between kernel exponential family estimation and MMD-GANs, revealing a new perspective for understanding GANs. Compared to the score matching based estimators, the proposed method improves both memory and time efficiency while enjoying stronger statistical properties, such as fully capturing smoothness in its statistical convergence rate while the score matching estimator appears to saturate. Finally, we show that the proposed estimator  empirically outperforms state-of-the-art methods in both kernel exponential family estimation and its conditional extension. 
\end{abstract}

%%%%%%%%%%%%%%%%%%%%%%%%%%%%%%%%%%%%%%%%%%%%%%%%%%%%%%%%%%%%
\section{Introduction}\label{sec:intro}
%%%%%%%%%%%%%%%%%%%%%%%%%%%%%%%%%%%%%%%%%%%%%%%%%%%%%%%%%%%%

% \setlength{\abovedisplayskip}{3pt}
% \setlength{\abovedisplayshortskip}{3pt}
% \setlength{\belowdisplayskip}{3pt}
% \setlength{\belowdisplayshortskip}{3pt}
% \setlength{\jot}{2pt}

% \setlength{\floatsep}{2ex}
% \setlength{\textfloatsep}{2ex}

The exponential family is one of the most important classes of distributions in statistics and machine learning. The exponential family possesses a number of useful properties~\citep{Brown86}and includes many commonly used distributions with finite-dimensional natural parameters, such as the Gaussian, Poisson and multinomial distributions, to name a few. It is natural to consider generalizing the richness and flexibility of the exponential family to an \emph{infinite}-dimensional parameterization via reproducing kernel Hilbert spaces~(RKHS)~\citep{CanSmo06}.

Maximum log-likelihood estimation~(MLE) has already been well-studied in the case of finite-dimensional exponential families, where desirable statistical properties such as asymptotic unbiasedness, consistency and asymptotic normality have been established. However, it is difficult to extend MLE to the infinite-dimensional case. Beyond the intractability of evaluating the partition function for a general exponential family, the necessary conditions for maximizing log-likelihood might not be feasible; that is, there may exist no solution to the KKT conditions in the infinite-dimensional case \citep{PisRog99,Fukumizu09}. To address this issue, \citet{BarShe91,GuQiu93,Fukumizu09} considered several ways to regularize the function space by constructing (a series of) finite-dimensional spaces that approximate the original RKHS, yielding a tractable estimator in the restricted finite-dimensional space. However, as~\citet{SriFukGreHyvetal17} note, even with the finite-dimension approximation, these algorithms are still expensive as every update requires Monte-Carlo sampling from a current model to compute the partition function.

An alternative score matching based estimator has recently been introduced by~\citet{SriFukGreHyvetal17}. This approach replaces the Kullback–Leibler~($KL$) with the Fisher divergence, defined by the expected squared distance between the score of the model~(\ie, the derivative of the $\log$-density) and the target distribution score~\citep{Hyvarinen05}. By minimizing the Tikhonov-regularized Fisher divergence, \citet{SriFukGreHyvetal17} develop a computable estimator for the infinite-dimensional exponential family that also obtains a consistency guarantee. Recently, this method has been generalized to \emph{conditional} infinite-dimensional exponential family estimation \citep{ArbGre17}. Although score matching avoids computing the generally intractable integral, it requires computing and saving the first- and second-order derivatives of the reproducing kernel for \emph{each} dimension on \emph{each} sample. For $n$ samples with $d$ features, this results in $\Ocal\rbr{n^2d^2}$ memory and $\Ocal\rbr{n^3d^3}$ time cost respectively, which becomes prohibitive for datasets with moderate sizes of $n$ and $d$. To alleviate this cost, \citet{SutStrArbArt17} utilize the $m$-rank Nyst{\"o}m approximation to the kernel matrix, reducing memory and time complexity to $\Ocal\rbr{nmd^2}$ and $\Ocal\rbr{m^3d^3}$ respectively. Although this reduces the cost dependence on sample size, the dependence on $d$ is unaffected, hence the estimator remains unsuitable for high-dimensional data. Estimating a general exponential family model by either MLE or score matching also generally requires some form of Monte-Carlo sampling, such as MCMC or HMC, to perform inference, which significantly adds to the computation cost.

In summary, both the MLE and score matching based estimators for the infinite-dimensional exponential family incur significant computational overhead, particularly in high-dimensional applications.

In this paper, we revisit penalized MLE for the kernel exponential family and propose a new estimation strategy. Instead of solving the log-likelihood equation directly, as in existing MLE methods, we exploit a \emph{doubly dual embedding} technique that leads to a novel saddle-point reformulation for the MLE (along with its conditional distribution generalization) in~\secref{sec:dual_mle}. We then propose a stochastic algorithm for the new view of penalized MLE in~\secref{sec:alg}. Since the proposed estimator is based on penalized MLE, it does not require the first- and second-order derivatives of the kernel, as in score matching, greatly reducing the memory and time cost. Moreover, since the saddle point reformulation of penalized MLE avoids the intractable integral, the need for a Monte-Carlo step is also bypassed, thus accelerating the learning procedure. This approach also learns a flexibly parameterized sampler simultaneously, therefore, further reducing inference cost. We present the consistency and rate of the new estimation strategy in the well-specified case, \ie, the true density belongs to the kernel exponential family, and an algorithm convergence guarantee in~\secref{sec:theory}. We demonstrate the empirical advantages of the proposed algorithm in~\secref{sec:experiments},  comparing to state-of-the-art estimators for both the kernel exponential family and its conditional extension~\citep{SutStrArbArt17,ArbGre17}.

%%%%%%%%%%%%%%%%%%%%%%%%%%%%%%%%%%%%%%%%%%%%%%%%%%%%%%%%%%%%%%%%%%%%%%%%%%%%%%%%%%%%%%%%%%%%
\section{Preliminaries}\label{sec:preliminary}
%%%%%%%%%%%%%%%%%%%%%%%%%%%%%%%%%%%%%%%%%%%%%%%%%%%%%%%%%%%%%%%%%%%%%%%%%%%%%%%%%%%%%%%%%%%%

We first provide a preliminary introduction to the exponential family and Fenchel duality, which will play vital roles in the derivation of the new estimator.

%%%%%%%%%%--------------------------------------------------------------------------------
\subsection{Exponential family}
%%%%%%%%%%--------------------------------------------------------------------------------
The natural form of the exponential family over $\Omega$ with the sufficient statistics $f\in \Fcal$ is defined as
\begin{equation}\label{eq:exp_family}
p_f\rbr{x} = p_0(x)\exp\rbr{\lambda f(x) - A\rbr{\lambda f}}, 
\end{equation}
where $A\rbr{\lambda f} \defeq \log\int_\Omega \exp\rbr{\lambda f\rbr{x}}p_0\rbr{x}dx$, $x\in \Omega\subset\RR^d,\, \lambda\in \RR$, and $\Fcal \defeq \cbr{f\in \Hcal: \exp\rbr{A\rbr{\lambda f}}< \infty}$. In this paper, we mainly focus on the case where $\Hcal$ is a reproducing Hilbert kernel space~(RKHS) with kernel $k\rbr{x, x'}$, such that $f\rbr{x} = \langle f, k\rbr{x, \cdot} \rangle$.  However, we emphasize that the proposed algorithm in~\secref{sec:alg} can be easily applied to arbitrary differentiable function parametrizations, such as deep neural networks.

Given samples $\Dcal = \sbr{x_i}_{i=1}^N$, a model with a finite-dimensional parameterization can be learned via maximum log-likelihood estimation~(MLE), 
\begin{eqnarray}\label{eq:mle}
\max_{f\in\Fcal}\frac{1}{N}\sum_{i=1}^N\log p_f(x_i) = \widehat\EE_{\Dcal}\sbr{\lambda f(x) + \log p_0\rbr{x}} - A(\lambda f), \nonumber
\end{eqnarray}
where $\widehat\EE_{\Dcal}\sbr{\cdot}$ denotes the empirical expectation over $\Dcal$. The MLE~\eq{eq:mle} is well-studied and has nice properties. However, the MLE is known to be ``ill-posed'' in the infinite-dimensional case, since the optimal solution might not be exactly achievable in the representable space. Therefore, penalized MLE has been introduced for the finite \citep{DudPhiSch07} and infinite dimensional \citep{GuQiu93,AltSmo06} exponential families respectively. Such regularization essentially relaxes the moment matching constraints  in terms of some norm, as shown later, guaranteeing the existence of a solution. In this paper, we will also focus on MLE with RKHS norm regularization.

One useful theoretical property of MLE for the exponential family is convexity w.r.t.\ $f$. With convex regularization, \eg, $\nbr{f}_{\Hcal}^2$, one can use stochastic gradient descent to recover a unique global optimum. Let $L\rbr{f}$ denote RKHS norm penalized log-likelihood, \ie,
\begin{equation}\label{eq:penalized_mle}
L\rbr{f} \defeq 
\frac{1}{N}\sum_{i=1}^N\log p_f(x_i) - \frac{\eta}{2}\nbr{f}_\Hcal^2.
\end{equation}
where $\eta>0$ denotes the regularization parameter. The gradient of $L\rbr{f}$ w.r.t. $f$ can be computed as
\begin{equation}\label{eq:grad}
\nabla_f L = \widehat\EE_{\Dcal}\sbr{\lambda \nabla_f f(x)} - \nabla_f A(\lambda f) - \eta f.
\end{equation}
To calculate the $\nabla_f A(\lambda f)$ in~\eq{eq:grad}, we denote $Z\rbr{\lambda f} = \int_\Omega \exp\rbr{\lambda f\rbr{x}}p_0\rbr{x}dx$ and expand the partition function by definition, 
\begin{eqnarray}
\nabla_f A(\lambda f) 
&=& \frac{1}{Z(\lambda f)}{\nabla_f \int_\Omega \exp\rbr{\lambda f(x)}p_0\rbr{x}dx} \nonumber \\
&=& \int_\Omega\frac{p_0\rbr{x}\exp\rbr{\lambda f(x)}}{Z(\lambda f)} {\nabla_f \lambda f(x)} dx\nonumber \\
&=& \EE_{p_f(x)}\sbr{\nabla_f \lambda f(x)}.
\end{eqnarray}
One can approximate the $\EE_{p_f(x)}\sbr{\nabla_f \lambda f(x)}$ by AIS~\citep{Neal05} or MCMC samples~\citep{VemGarBol09}, which leads to the Contrastive Divergence~(CD) algorithm~\citep{Hinton02}.  

To avoid costly MCMC sampling in estimating the gradient,  \citet{SriFukGreHyvetal17} construct an estimator based on score matching instead of MLE, which minimizes the penalized Fisher divergence. Plugging the kernel exponential family into the empirical Fisher divergence, the optimization reduces to 
\begin{equation}\label{eq:score_matching}
J\rbr{f} \defeq \frac{\lambda^2}{2}\langle f, \Chat f \rangle_\Hcal + \lambda\langle f,  \hat\delta\rangle_\Hcal + \frac{\eta}{2}\nbr{f}^2_{\Hcal},
\end{equation}
where 
\begin{eqnarray*}
\Chat&\defeq& \frac{1}{n} \sum_{i=1}^n\sum_{j=1}^d \partial_j k\rbr{x_i, \cdot}\otimes \partial_j k\rbr{x_i, \cdot},\\[-1mm]
\hat\delta &\defeq& \frac{1}{n}\sum_{i=1}^n\sum_{j=1}^d \partial_jk\rbr{x_i, \cdot}\rbr{\partial_j\log p_0\rbr{x_i} + \partial_j^2 k\rbr{x_i, \cdot}}. 
\end{eqnarray*}
As we can see, the score matching objective~\eqref{eq:score_matching} is convex and does not involve the intractable integral in $A\rbr{\lambda f}$. However, such an estimator requires the computation of first- and second-order of derivatives of kernel for each dimension on each data, leading to memory and time cost of $\Ocal\rbr{n^2d^2}$ and $\Ocal\rbr{n^3d^3}$ respectively. This quickly becomes prohibitive for even modest $n$ and $d$.

The same difficulty also appears in the score matching based estimator in~\citep{ArbGre17} for the \emph{conditional exponential family}, which is defined as 
\begin{equation}\label{eq:conditional_exp_family}
p\rbr{y|x} = p_0\rbr{y}\exp\rbr{\lambda f\rbr{x, y} - A_x\rbr{\lambda f}},\,\, f\in\Fcal
\end{equation}
where $y\in\Omega_y\subset \RR^p$, $x\in\Omega_x\subset \RR^d$, $\lambda \in \RR$, and $A_x\rbr{\lambda f}\defeq \log \int_{\Omega_y} p_0(y)\exp\rbr{\lambda f\rbr{x, y}}dy$, $\Fcal$ is defined as $\cbr{f\in \Hcal_y: \exp\rbr{A_x\rbr{\lambda f}}< \infty}$. We consider $f:\Omega_x\times\Omega_y\rightarrow \RR$ such that $f\rbr{x, \cdot}$ is in RKHS $\Hcal_y$ for $\forall x\in \Omega_x$. Denoting $\Tcal\in\Hcal_{\Omega_x} :\Omega_x\rightarrow \Hcal_y$ such that $\Tcal_x\rbr{y} = f\rbr{x, y}$, we can derive its kernel function following~\citet{MicPon05,ArbGre17}. By the Riesz representation theorem, $\forall x\in\Omega_x$ and $h\in\Hcal_y$, there exists a linear operator $\Gamma_x:\Hcal_y\rightarrow \Hcal_{\Omega_x}$ such that
$$
\langle h, \Tcal_x \rangle_{\Hcal_y} = \langle \Tcal, \Gamma_xh \rangle_{\Hcal_{\Omega_x}}, \,\, \forall\Tcal\in\Hcal_{\Omega_x}.
$$
Then, the kernel can be defined by composing $\Gamma_x$ with its dual, \ie, $k\rbr{x, x'} = \Gamma_x^*\Gamma_{x'}$ and the function $f\rbr{x, y} = \Tcal_x\rbr{y} = \langle \Tcal, \Gamma_x k\rbr{y, \cdot} \rangle$. We follow such assumptions for kernel conditional exponential family.

%%%%%%%%%%--------------------------------------------------------------------------------
\subsection{Convex conjugate and Fenchel duality}
%%%%%%%%%%--------------------------------------------------------------------------------
Denote $h\rbr{\cdot}$ as a function $ \RR^d\rightarrow \RR$, then its convex conjugate function is defined as
$$
h^*(u) = \sup_{v\in \RR^d}\{u^\top v - h(v)\}.
$$
If $h\rbr{\cdot}$ is proper, convex and lower semicontinuous, the conjugate function, $h^*\rbr{\cdot}$, is also proper, convex and lower semicontinuous. Moreover, $h$ and $h^*$ are dual to each other, \ie, $\rbr{h^*}^* = h$. Such a relationship is known as Fenchel duality~\citep{Rockafellar70,HirLem12}. By the conjugate function, we can represent the $h$ by as,
$$
h(v) = \sup_{u\in \RR^d}\{v^\top u - h^*(u)\}.
$$
The supremum achieves if $v\in \partial h^*(u)$, or equivalently $u\in \partial h(v)$.

%%%%%%%%%%%%%%%%%%%%%%%%%%%%%%%%%%%%%%%%%%%%%%%%%%%%%%%%%%%%%%%%%%%%%%%%%%%%%%%%%%%%%%%%%%%%
\section{A Saddle-Point Formulation of Penalized MLE}\label{sec:dual_mle}
%%%%%%%%%%%%%%%%%%%%%%%%%%%%%%%%%%%%%%%%%%%%%%%%%%%%%%%%%%%%%%%%%%%%%%%%%%%%%%%%%%%%%%%%%%%%
As discussed in~\secref{sec:preliminary}, the penalized MLE for the exponential family involves computing the $\log$-partition functions, $A\rbr{\lambda f}$ and $A_x\rbr{\lambda f}$, 
which are intractable in general. In this section, we first introduce a saddle-point reformulation of the penalized MLE of the exponential family, using Fenchel duality to bypass the computation of the intractable log-partition function. This approach can also be generalized to the conditional exponential family. First, observe that we can rewrite the log-partition function $A\rbr{\lambda f}$ via Fenchel duality as follows. 
\begin{theorem}[Fenchel dual of log-partition]\label{thm:fenchel_log_partition}
\begin{eqnarray}
A\rbr{\lambda f} &=& \max_{q\in \Pcal}\,\, \lambda\inner{q(x)}{f(x)}_2 - KL\rbr{q||p_0},\quad \label{eq:Af}\\
p_f\rbr{x} &=& \argmax_{q\in \Pcal}\,\, \lambda\inner{q(x)}{f(x)}_2 - KL\rbr{q||p_0},\quad\label{eq:pf}
\end{eqnarray}
where $\inner{f}{g}_2\defeq \int_\Omega f\rbr{x}g\rbr{x}dx$, $\Pcal$ denotes the space of distributions with bounded $L_2$ norm and $KL\rbr{q||p_0}\defeq \int_\Omega q\rbr{x}\log\frac{q\rbr{x}}{p_0\rbr{x}}dx$. 
\end{theorem}
\begin{proof}
Denote $l\rbr{q} \defeq \lambda\inner{q(x)}{f\rbr{x}} - KL\rbr{q||p_0}$, which is strongly concave w.r.t. $q \in \Pcal$, the optimal $q^*$ can be obtained by setting the 
$$
\log q^*(x)\propto \lambda f\rbr{x}+\log{p_0(x)}. 
$$
Since $q^*\in\Pcal$, we have 
$$
q^*(x) = p_0(x)\exp\rbr{\lambda f(x) - A\rbr{\lambda f}} = p_f(x),
$$
which leads to (\ref{eq:pf}). Plugging $q^*$ to $l(q)$, we obtain the maximum as $\log \int_\Omega \exp\rbr{\lambda f(x)}p_0(x)dx$, which is exactly $A\rbr{\lambda f}$, leading to (\ref{eq:Af}).
\end{proof}
Therefore, invoking the Fenchel dual of $A\rbr{\lambda f}$ into the penalized MLE, we achieve a saddle-point optimization, \ie, 
{\small
\begin{eqnarray}\label{eq:primal_CD}
\max_{f\in\Fcal} L\rbr{f} \propto \min_{q\in\Pcal}\,\, \underbrace{\widehat\EE_{\Dcal}\sbr{f(x)} - \EE_{q(x)}\sbr{f(x)} - \frac{\eta}{2}\nbr{f}_\Hcal^2+ \frac{1}{\lambda}KL\rbr{q||p_0}}_{\ell(f, q)}. 
\end{eqnarray}
}

The min-max reformulation of the penalized MLE in~\eqref{eq:primal_CD} resembles the optimization in MMD GAN~\citep{LiChaYuYanetal17,BinSutArbArt18}. In fact, the dual problem of the penalized MLE of the exponential family is a \emph{KL-regularized} MMD GAN with a special design of the kernel family. Alternatively, if $f$ is a Wasserstein-$1$ function, the optimization resembles the Wasserstein GAN~\citep{ArjChiBot17}. 

Next, we consider the duality properties.
\begin{theorem}[weak and strong duality]\label{thm:minmax_switch} The weak duality holds in general, \ie,
$$
\max_{f\in\Fcal}\min_{q\in\Pcal}\,\, \ell(f, q) \le \min_{q\in\Pcal}\max_{f\in\Fcal} \,\,\ell(f, q).
$$
% If $\Fcal$ and $\Pcal$ convex, \eg, $\Fcal$ is {some bounded RKHS with continuous kernel on the compact domain}, 
The strong duality holds when $\Fcal$ is a closed RKHS and $\Pcal$ is the distributions with bounded $L_2$ norm, \ie, 
\begin{equation}\label{eq:strong_dual}
\max_{f\in\Fcal}\min_{q\in\Pcal}\,\, \ell(f, q) = \min_{q\in\Pcal}\max_{f\in\Fcal} \,\,\ell(f, q).
\end{equation}
\end{theorem}
\thmref{thm:minmax_switch} can be obtained by directly applying the minimax theorem~\citep{EkeTem99}[Proposition 2.1]. We refer to the $\max$-$\min$ problem in~\eqref{eq:strong_dual} as the primal problem, while the $\min$-$\max$ form as the its dual form. 

\paragraph{Remark (connections to MMD GAN):} Consider the dual problem with the kernel learning, \ie,
\begin{equation}\label{eq:dual_cd}
\min_{q\in\Pcal}\max_{f\in\Fcal_\phi, \phi} \widehat\EE_{\Dcal}\sbr{f(x)} - \EE_{q(x)}\sbr{f(x)} - \frac{\eta}{2}\nbr{f}_\Hcal^2+ \frac{1}{\lambda}KL\rbr{q||p_0}.
\end{equation}
where we involve the parameters of the kernel $\phi$ to be learned in the optimization. By setting the gradient $\nabla_f \ell\rbr{f, q} = 0$, for fixed $q$, we obtain the optimal witness function in the RKHS, $f_q = \frac{1}{\eta}\rbr{\widehat\EE\sbr{k_\phi\rbr{x, \cdot}} - \EE_{q}\sbr{k_\phi\rbr{x, \cdot}}}$, which leads to
\begin{eqnarray}\label{eq:reg_mmd}
\min_{q\in\Pcal}\max_{\phi}\underbrace{\widehat\EE\sbr{k_\phi\rbr{x, x'}}-2\widehat\EE\EE_q\sbr{k_\phi\rbr{x, x'}} + \EE_{q}\sbr{k_\phi\rbr{x,x'}}}_{MMD_\phi\rbr{\Dcal, q}}+ \frac{2\eta}{\lambda}KL\rbr{q||p_0}.
\end{eqnarray}
This can be regarded as the $KL$-divergence regularized MMD GAN. Thus, with $KL$-divergence regularization, the MMD GAN learns an infinite-dimension exponential family in an adaptive RKHS. Such a novel perspective bridges GAN and exponential family estimation, which appears to be of independent interest and potentially brings a new connection to the GAN literature for further theoretical development. 
\paragraph{Remark (connections to Maximum Entropy Moment Matching):} \citet{AltSmo06,DudPhiSch07} discuss the maximum entropy moment matching method for distribution estimation, 
\begin{eqnarray}\label{eq:max_ent}
\min_{q\in \Pcal}\,\, && KL\rbr{q||p_0}\\
\st &&\nbr{\EE_q\sbr{k\rbr{x, \cdot}} - \widehat\EE\sbr{k\rbr{x,\cdot}}}_{\Hcal_k}\le \frac{\eta'}{2},\nonumber 
\end{eqnarray} 
whose dual problem will be reduced to the penalized MLE~\eqref{eq:penalized_mle} with proper choice of $\eta'$~\citep{AltSmo06}[Lemma 6]. Interestingly, the proposed saddle-point formulation~\eqref{eq:primal_CD} shares the solution to~\eqref{eq:max_ent}. From the maximum entropy view,  the penalty $\nbr{f}_{\Hcal}$ relaxes the moment matching constraints. However, the algorithms provided in~\citet{AltSmo06,DudPhiSch07} simply ignore the difficulty in computing the expectation in $A\rbr{\lambda f}$, which is not practical, especially when $f$ is infinite-dimensional.

Similar to~\thmref{thm:fenchel_log_partition}, we can also represent $A_x\rbr{\lambda f}$ by its Fenchel dual,
\begin{eqnarray}\label{eq:fenchel_dual_condition}
A_x\rbr{\lambda f} 
&=& \max_{q(\cdot|x)\in \Pcal} \lambda\inner{q(y|x)}{f\rbr{x,y}}- KL\rbr{q||p_0}, \\
p_f\rbr{y|x} &=& \argmax_{q(\cdot|x)\in \Pcal} \lambda\inner{q(x)}{f(x)} - KL\rbr{q||p_0}.
\end{eqnarray}
Then, we can recover the penalized MLE for the conditional exponential 
family as 
\begin{eqnarray}\label{eq:primal_conditional_CD}
\max_{f\in\Fcal}\min_{q\rbr{\cdot|x}\in\Pcal}\,\,\widehat\EE_{\Dcal}\sbr{f(x,y)} - \EE_{q(y|x)}\sbr{f(x, y)} - \frac{\eta}{2}\nbr{f}_{\Hcal}^2 + \frac{1}{\lambda}KL\rbr{q||p_0}. \nonumber
\end{eqnarray}

The saddle-point reformulations of penalized MLE in~\eqref{eq:primal_CD} and~\eqref{eq:primal_conditional_CD} bypass the difficulty in the partition function. Therefore, it is very natural to consider learning the exponential family by solving the saddle-point problem~\eqref{eq:primal_CD} with an appropriate parametrized dual distribution $q\rbr{x}$. This approach is referred to as the ``dual embedding'' technique in~\citet{DaiHePanBooetal16}, which requires:
\begin{itemize}
  \item[{\bf i)}] the parametrization family should be flexible enough to reduce the extra approximation error;  
  \item[{\bf ii)}] the parametrized representation should be able to provide density value. 
\end{itemize}
As we will see in~\secref{sec:theory}, the flexibility of the parametrization of the dual distribution will have a significant effect on the consistency of the estimator. 

One can of course use the kernel density estimator~(KDE) as the dual distribution parametrization, which preserves convex-concavity. However, KDE will easily fail when approximating high-dimensional data. Applying the reparametrization trick with a suitable class of probability distributions~\citep{KinWel13,RezMohWie14} is another alternative to parametrize the dual distribution. However, the class of such parametrized distributions is typically restricted to simple known distributions, which might not be able to approximate the true solution, potentially leading to a  huge approximation bias. At the other end of the spectrum, the distribution family generated by transport mapping is sufficiently flexible to model smooth distributions~\citep{GooPouMirXuetal14,ArjChiBot17}. However, the density value of such distributions, \ie, $q\rbr{x}$, is not available for $KL$-divergence computation, and thus, is not applicable for parameterizing the dual distribution. Recently, flow-based parametrized density functions~\citep{RezMoh15,KinSalJozCheetal16,DinSohBen16} have been proposed for a trade-off between the flexibility and tractability. However, the expressive power of existing flow-based models remains restrictive even in our synthetic example and cannot be directly applied to conditional models.

%%%%%%%%%%--------------------------------------------------------------------------------
\subsection{Doubly Dual Embedding for MLE}
%%%%%%%%%%--------------------------------------------------------------------------------

Transport mapping is very flexible for generating smooth distributions. However, a major difficulty is that it lacks the ability to obtain the density value $q\rbr{x}$, making the computation of the $KL$-divergence impossible. To retain the flexibility of transport mapping and avoid the calculation of $q\rbr{x}$, we introduce \emph{doubly dual embedding}, which  achieves a delicate balance between flexibility and tractability. 
  
First, noting that $KL$-divergence is also a convex function, we consider the Fenchel dual of $KL$-divergence~\citep{NguWaiJor08}, \ie, 
\begin{eqnarray}\label{eq:kl_dual}
KL\rbr{q||p_0} &=& \max_{\nu} \EE_{q}\sbr{\nu(x)} - \EE_{p_0}\sbr{\exp\rbr{\nu(x)}} + 1,\\
\log\frac{q\rbr{x}}{p_0\rbr{x}} &=& \argmax_{\nu} \EE_{q}\sbr{\nu(x)} - \EE_{p_0}\sbr{\exp\rbr{\nu(x)}} + 1,
\end{eqnarray}
with $\nu\rbr{\cdot}: \Omega\rightarrow \RR$. One can see in the dual representation of the $KL$-divergence that the introduction of the auxiliary optimization variable $\nu$ eliminates the explicit appearance of $q\rbr{x}$ in~\eqref{eq:kl_dual}, which makes the transport mapping parametrization for $q\rbr{x}$  applicable. Since there is no extra restriction on $\nu$, we can use arbitrary smooth function approximation for $\nu\rbr{\cdot}$, such as kernel functions or neural networks. 

Applying the dual representation of $KL$-divergence into the dual problem of the saddle-point reformulation~\eqref{eq:primal_CD}, we obtain the ultimate saddle-point optimization for the estimation of the kernel exponential family,
\begin{equation}\label{eq:double_dual_mle}
\begin{split}
\min_{q\in\Pcal}\max_{f, \nu\in\Fcal}\tilde\ell\rbr{f, \nu, q}\defeq \widehat\EE_{\Dcal}\sbr{f} - \EE_q\sbr{f} -\frac{\eta}{2}\nbr{f}_{\Hcal}^2 +\frac{1}{\lambda }\rbr{\EE_q\sbr{\nu} - \EE_{p_0}\sbr{\exp\rbr{\nu}}}. 
\end{split}
\end{equation}
Note that several other work also applied Fenchel duality to $KL$-divergence~\citep{NguWaiJor08,NowCseTom16}, but the approach taken here is different as it employs dual of $KL\rbr{q||p_0}$, rather than $KL\rbr{q||\Dcal}$. 

\begin{algorithm}[t]
\caption{{\bf Doubly Dual Embedding-SGD} for saddle-point reformulation of MLE~\eqref{eq:double_dual_mle}}
  \begin{algorithmic}[1]\label{alg:sgd_double_dual}
    \FOR{$l=1,\ldots, L$}
      \STATE Compute $\rbr{f, \nu}$ by~\algref{alg:sgd_f}.
      \STATE Sample $\cbr{\xi_0\sim p\rbr{\xi}}_{b=1}^B$.
      \STATE Generate $x_b = g_{w_g}\rbr{\xi}$ for $b=1,\ldots, B$.
      \STATE Compute stochastic approximation $\widehat\nabla_{w_g} \Lhat \rbr{w_g}$ by~\eqref{eq:dual_gradient}.
      \STATE Update $w_g^{l+1} = w_g^{l} - \rho_l \widehat\nabla_{w_g} L\rbr{w_g}$.
    \ENDFOR
    \STATE Output $w_g$ and $f$.
  \end{algorithmic}
\end{algorithm}

\paragraph{Remark (Extension for kernel conditional exponential family):} Similarly, we can apply the doubly dual embedding technique to the penalized MLE of the kernel conditional exponential family,
which leads to 
\begin{eqnarray}\label{eq:double_dual_conditional_mle}
\min_{q\in\Pcal_x}\max_{f, \nu\in\Fcal} \widehat\EE_{\Dcal}\sbr{f} - \EE_{q\rbr{y|x}, x\sim\Dcal}\sbr{f} -\frac{\eta}{2}\nbr{f}_{\Hcal}^2 +\frac{1}{\lambda}\rbr{\EE_{q\rbr{y|x}, x\sim\Dcal}\sbr{\nu} - \EE_{p_0\rbr{y}}\sbr{\exp\rbr{\nu}}}.\nonumber
\end{eqnarray}

In summary, with the \emph{doubly dual embedding} technique, we derive a saddle-point reformulation of penalized MLE that bypasses the difficulty of handling the intractable partition function while allowing great flexibility in parameterizing the dual distribution.

%%%%%%%%%%%%%%%%%%%%%%%%%%%%%%%%%%%%%%%%%%%%%%%%%%%%%%%%%%%%%%%%%%%%%%%%%%%%%%%%%%%%%%%%%%%%
\section{Practical Algorithm}\label{sec:alg}
%%%%%%%%%%%%%%%%%%%%%%%%%%%%%%%%%%%%%%%%%%%%%%%%%%%%%%%%%%%%%%%%%%%%%%%%%%%%%%%%%%%%%%%%%%%%

\begin{algorithm}[t]
\caption{Stochastic Functional Gradients for $f$ and $\nu$}
  \begin{algorithmic}[1]\label{alg:sgd_f}
    \FOR{$k=1,\ldots, K$}
      \STATE Sample ${\xi\sim p\rbr{\xi}}$.
      \STATE Generate $x = g\rbr{\xi}$.
      \STATE Sample ${x'\sim p_0\rbr{x}}$.
      \STATE Compute stochastic function gradient w.r.t. $f$ and $\nu$ with~\eqref{eq:f_grad} and~\eqref{eq:nu_grad}.
      \STATE Update $f_k$ and $\nu_k$ with~\eqref{eq:f_update} and~\eqref{eq:nu_update}.
    \ENDFOR
    \STATE Output $\rbr{f_K, \nu_K}$ .
  \end{algorithmic}
\end{algorithm}
In this section, we introduce the transport mapping parametrization for the dual distribution, $q\rbr{x}$, and apply the stochastic gradient descent for solving the optimization problems in \eqref{eq:double_dual_mle} and \eqref{eq:double_dual_conditional_mle}. For simplicity of exposition, we only illustrate the algorithm for~\eqref{eq:double_dual_mle}. The algorithm can be easily applied to~\eqref{eq:double_dual_conditional_mle}.

Denote the parameters in the transport mapping for $q\rbr{x}$ as $w_g$, such that $\xi\sim p\rbr{\xi}$ and $x = g_{w_g}\rbr{\xi}$. We illustrate the algorithm with a kernel parametrize $\nu\rbr{\cdot}$. There are many alternative choices of parametrization of $\nu$, \eg, neural networks---the proposed algorithm is still applicable to the parametrization as long as it is differentiable. We abuse notation somewhat by using $\tilde\ell\rbr{f, \nu, w_g}$ as $\tilde\ell\rbr{f, \nu, q}$ in~\eqref{eq:double_dual_mle}. 
with such parametrization, the~\eqref{eq:double_dual_mle} is an upper bound of the penalty MLE in general by~\thmref{thm:minmax_switch}. 

With the kernel parametrized $\rbr{f, \nu}$, the inner maximization over $f$ and $\nu$ is a standard concave optimization. We can solve it using existing algorithms to achieve the global optimal solution. Due to the existence of the expectation, we will use the stochastic functional gradient descent for scalability. Given $\rbr{f, \nu}\in\Hcal$, following the definition of functional gradients~\citep{KivSmoWil04,DaiXieHeLiaEtAl14}, we have
\begin{eqnarray}\label{eq:f_grad}
\zeta_f\rbr{\cdot}\defeq\nabla_f \tilde\ell\rbr{f, \nu, w_g} = \widehat\EE_{\Dcal}\sbr{k\rbr{x, \cdot}} - \EE_{\xi}\sbr{k\rbr{g_{w_g}\rbr{\xi_0}, \cdot}} -\eta f\rbr{\cdot},
\end{eqnarray}
\begin{eqnarray}\label{eq:nu_grad}
\zeta_\nu\rbr{\cdot}\defeq\nabla_\nu \tilde\ell\rbr{f, \nu, w_g} = \frac{1}{\lambda}\rbr{\EE_\xi\sbr{k\rbr{g_{w_g}\rbr{\xi}, \cdot}} - \EE_{p_0}\sbr{\exp\rbr{\nu\rbr{x}}k\rbr{x, \cdot}}}. 
\end{eqnarray}
In the $k$-th iteration, given sample $x\sim\Dcal, \xi\sim p\rbr{\xi}$, and $x'\sim p_0\rbr{x}$, the update rule for $f$ and $\nu$ will be
\begin{eqnarray}\label{eq:f_update}
f_{k+1}\rbr{\cdot} =\rbr{1 - \eta\tau_k} f_k\rbr{\cdot} + {\tau_k}\rbr{k\rbr{x, \cdot} - k\rbr{g_{w_g}\rbr{\xi}, \cdot}},
\end{eqnarray}
\begin{eqnarray}\label{eq:nu_update}
\nu_{k+1}\rbr{\cdot} =\nu_k\rbr{\cdot} + \frac{\tau_k}{\lambda }\rbr{k\rbr{g_{w_g}\rbr{\xi}, \cdot} - \exp\rbr{\nu_k\rbr{x'}}k\rbr{x', \cdot}},\nonumber
\end{eqnarray}
where $\tau_k$ denotes the step-size. 

Then, we consider the update rule for the parameters in dual transport mapping embedding. 
\begin{theorem}[Dual gradient]\label{thm:dual_gradient}
Denoting  $\rbr{f_{w_g}^*, \nu_{w_g}^*} = \argmax_{\rbr{f, \nu}\in\Hcal} \tilde \ell\rbr{f, \nu, w_g}$ and $\Lhat\rbr{w_g} = \tilde\ell\rbr{f_{w_g}^*, \nu_{w_g}^*, w_g}$, we have 
\begin{eqnarray}\label{eq:dual_gradient}
\nabla_{w_g} \Lhat\rbr{w_g} = -\EE_{\xi}\sbr{\nabla_{w_g} f_{w_g}^*\rbr{g_{w_g}\rbr{\xi}}} + \frac{1}{\lambda}\EE_{\xi}\sbr{\nabla_{w_g} \nu_{w_g}^*\rbr{g_{w_g}\rbr{\xi}}}.
\end{eqnarray}
\end{theorem}
Proof details are given in \appref{appendix:subsec:dual_grad}. With this gradient estimator, we can apply stochastic gradient descent to update $w_g$ iteratively. We summarize the updates for $\rbr{f,\nu}$ and $w_g$ in~\algref{alg:sgd_double_dual} and~\ref{alg:sgd_f}.

Compared to score matching based estimators~\citep{SriFukGreHyvetal17,SutStrArbArt17}, although the convexity no longer holds for the saddle-point estimator, the doubly dual embedding estimator avoids representing $f$ with derivatives of the kernel, thus avoiding the memory cost dependence on the square of dimension. In particular, the proposed estimator for $f$ based on~\eqref{eq:double_dual_mle} reduces the memory cost from $\Ocal\rbr{n^2d^2}$ to $\Ocal\rbr{K^2}$ where $K$ denotes the number of iterations in~\algref{alg:sgd_f}. In terms of time cost, we exploit the stochastic update, which is naturally suitable for large-scale datasets and avoids the matrix inverse computation in score matching estimator whose cost is $\Ocal\rbr{n^3d^3}$. We also learn the dual distribution simultaneously, which can be easily used to generate samples from the exponential family for inference, thus saving the cost of Monte-Carlo sampling in the inference stage. For detailed discussion about the computation cost, please refer to~\appref{appendix:comp_cost}. 

\paragraph{Remark (random feature extension):} Memory cost is a well-known bottleneck for applying kernel methods to large-scale problems. When we set $K=n$, the memory cost will be $\Ocal\rbr{n^2}$, which is prohibitive for millions of data points. Random feature approximation~\citep{RahRec08,DaiXieHeLiaEtAl14,Bach15} can be utilized for scaling up kernel methods. The proposed~\algref{alg:sgd_double_dual} and~\algref{alg:sgd_f} are also compatible with random feature approximation, and hence, applicable to large-scale problems in the same way. With $r$ random features, we can further reduce the memory cost of storing $f$ to $\Ocal\rbr{rd}$. However, even with a random feature approximation, the score matching based estimator will still require $\Ocal\rbr{rd^2}$ memory. One can also learn random features by back-propagation, which leads to the neural networks extension. Please refer to the details of this variant of~\algref{alg:sgd_double_dual} in Appendix~\ref{appendix:random_feature}.

%%%%%%%%%%%%%%%%%%%%%%%%%%%%%%%%%%%%%%%%%%%%%%%%%%%%%%%%%%%%%%%%%%%%%%%%%%%%%%%%%%%%%%%%%%%%
\section{Theoretical Analysis}\label{sec:theory}
%%%%%%%%%%%%%%%%%%%%%%%%%%%%%%%%%%%%%%%%%%%%%%%%%%%%%%%%%%%%%%%%%%%%%%%%%%%%%%%%%%%%%%%%%%%%

In this section, we will first provide the analysis of consistency and the sample complexity of the proposed estimator based on the saddle-point reformulation of the penalized MLE in the well-specified case, where the true density is assumed to be in the kernel (conditional) exponential family, following~\citet{SutStrArbArt17,ArbGre17}. Then, we consider the convergence property of the proposed algorithm. We mainly focus on the analysis for $p\rbr{x}$. The results can be easily extended to kernel conditional exponential family $p\rbr{y|x}$. 

%%%%%%%%%%--------------------------------------------------------------------------------
\subsection{Statistical Consistency}
%%%%%%%%%%--------------------------------------------------------------------------------

We explicitly consider approximation error from the dual embedding in the consistency of the proposed estimator. We first establish some notation that will be used in the analysis. For simplicity, we set $\lambda =1$ and $p_0\rbr{x} = 1$ improperly in the exponential family expression~\eqref{eq:exp_family}. We denote $p^*\rbr{x} = \exp\rbr{f^*\rbr{x} - A\rbr{f^*}}$ as the ground-true distribution with its true potential function $f^*$ and $\rbr{\ftil, \qtil, \vtil}$ as the optimal primal and dual solution to the saddle point reformulation of the penalized MLE~\eqref{eq:double_dual_mle}. The parametrized dual space is denoted as $\Pcal_w$. We denote $p_{\ftil} \defeq \exp\rbr{\lambda \ftil - A\rbr{\lambda \ftil}}$ as the exponential family generated by $\ftil$. We have the consistency results as
\begin{theorem}\label{thm:consistency}
Assume the spectrum of kernel $k\rbr{\cdot, \cdot}$ decays sufficiently homogeneously in rate $l^{-r}$. With some other mild assumptions listed in~\appref{appendix:subsec:consistency}, we have as $\eta\rightarrow 0$ and $n\eta^{\frac{1}{r}}\to\infty$,
\begin{eqnarray*}
KL\rbr{p^*||p_{\ftil}} + KL\rbr{p_{\ftil}||p^*} = \Ocal_{p^*}\rbr{n^{-1}\eta^{-\frac{1}{r}} + \eta + \epsilon_{approx}},
\end{eqnarray*}
where $\epsilon_{approx}\defeq \sup_{f\in\Fcal}\inf_{q\in\Pcal_w}KL\rbr{q||p_{f}}$ denotes the approximate error due to the parametrization of $\qtil$ and $\tilde\nu$.
% $\epsilon_{approx}\defeq \sup_{f\in \Fcal}\inf_{q\in\Pcal_w} \nbr{p_f -q}_{p^*}$. 
Therefore, when setting $\eta = \Ocal\rbr{n^{-\frac{r}{1+r}}}$, $p_{\ftil}$ converge
s to $p^*$  in terms of Jensen-Shannon divergence at rate 
$
\Ocal_{p^*}\rbr{n^{-\frac{r}{1+r}} + \epsilon_{approx}}.
$
\end{theorem}
For the details of the assumptions and the proof, please refer to~\appref{appendix:subsec:consistency}. Recall the connection between the proposed model and MMD GAN as discussed in~\secref{sec:dual_mle}. \thmref{thm:consistency} also provides a learnability guarantee for a class of GAN models as a byproduct. The most significant difference of the bound provided above, compared to~\citet{GuQiu93,AltSmo06}, is the explicit consideration of the bias from the dual parametrization. Moreover, instead of the Rademacher complexity used in the sample complexity results of~\citet{AltSmo06}, our result exploits the spectral decay of the kernel, which is more directly connected to properties of the RKHS.

From~\thmref{thm:consistency} we can clearly see the effect of the parametrization of the dual distribution: if the parametric family of the dual distribution is simple, the optimization for the saddle-point problem may become easy, however, $\epsilon_{approx}$ will dominate the error. The other extreme case is to also use the kernel exponential family to parametrize the dual distribution and the parametric family of $\nu$ contains $\log\frac{p_f\rbr{x}}{p_0\rbr{x}}$, then, $\epsilon_{approx}$ will reduce to $0$, however, the optimization will be difficult to handle. The saddle-point reformulation provides us the opportunity to balance the difficulty of the optimization 
with approximation error. 

The statistical consistency rate of our estimator and the score matching estimator~\citep{SriFukGreHyvetal17} are derived under different assumptions, therefore, they are not directly comparable. However, since the smoothness is not fully captured by the score matching based estimator in~\citet{SriFukGreHyvetal17}, it only achieves $\Ocal\rbr{n^{-\frac{2}{3}}}$ even if $f$ is infinitely smooth. While under the case that dual distribution parametrization is relative flexible, \ie, $\epsilon_{approx}$ is negligible, and the the spectrum of the kernel decay rate $r\rightarrow\infty$, the proposed estimator will converge in rate $\Ocal\rbr{n^{-1}}$, which is significantly more efficient than the score matching method.

%%%%%%%%%%--------------------------------------------------------------------------------
\subsection{Algorithm Convergence}
%%%%%%%%%%--------------------------------------------------------------------------------

It is well-known that the stochastic gradient descent converges for saddle-point problem with convex-concave property~\citep{NemJudLanSha09}. However, for better the dual parametrization to reduce $\epsilon_{approx}$ in~\thmref{thm:consistency}, we parameterize the dual distribution with the nonlinear transport mapping, which breaks the convexity. In fact, by~\thmref{thm:dual_gradient}, we obtain the unbiased gradient w.r.t. $w_g$. Therefore, the proposed~\algref{alg:sgd_double_dual} can be understood as applying the stochastic gradient descent for the non-convex dual minimization problem, \ie, $\min_{w_g} \Lhat\rbr{w_g}\defeq \tilde\ell\rbr{f_{w_g}^*, \nu_{w_g}^*, w_g}$. From such a view, we can prove the sublinearly convergence rate to a stationary point when stepsize is diminishing following~\citet{GhaLan13,DaiShaLiXiaHeetal17}. We list the result below for completeness.

\begin{theorem}\label{thm:convergence_opt}
Assume that the parametrized objective $\Lhat\rbr{w_g}$ is $C$-Lipschitz and variance of its stochastic gradient is bounded by $\sigma^2$. Let the algorithm run for $L$ iterations with stepsize $\rho_l=\min\{\frac{1}{L}, \frac{D'}{\sigma\sqrt{L}}\}$ for some $D'>0$ and output $w_g^1,\ldots, w_g^L$. Setting the candidate solution to be $\widehat w_g$ randomly chosen from $w_g^1,\ldots, w_g^L$ such that $P(w=w_g^j)=\frac{2\rho_j-C\rho_j^2}{\sum_{j=1}^L(2\rho_j-C\rho_j^2)}$, then it holds that
$\EE\sbr{\nbr{\nabla \Lhat(\widehat w_g)}^2}\leq \frac{C D^2}{L}+ (D'+\frac{D}{D'})\frac{\sigma}{\sqrt{L}}$
where $D:=\sqrt{2(\Lhat(w_g^1) -\min\Lhat(w_g))/L}$ represents the distance of the initial solution to the optimal solution. 
\end{theorem}
The above result implies that under the choice of the parametrization of $f, \nu$ and $g$, the proposed~\algref{alg:sgd_double_dual} converges sublinearly to a stationary point, whose rate will depend on the smoothing parameter.

%%%%%%%%%%%%%%%%%%%%%%%%%%%%%%%%%%%%%%%%%%%%%%%%%%%%%%%%%%%%%%%%%%%%%%%%%%%%%%%%%%%%%%%%%%%%
\section{Experiments}\label{sec:experiments}
%%%%%%%%%%%%%%%%%%%%%%%%%%%%%%%%%%%%%%%%%%%%%%%%%%%%%%%%%%%%%%%%%%%%%%%%%%%%%%%%%%%%%%%%%%%%

%
\begin{figure}[t]
  \begin{tabular}{ccc}
    \includegraphics[width=0.315\textwidth]{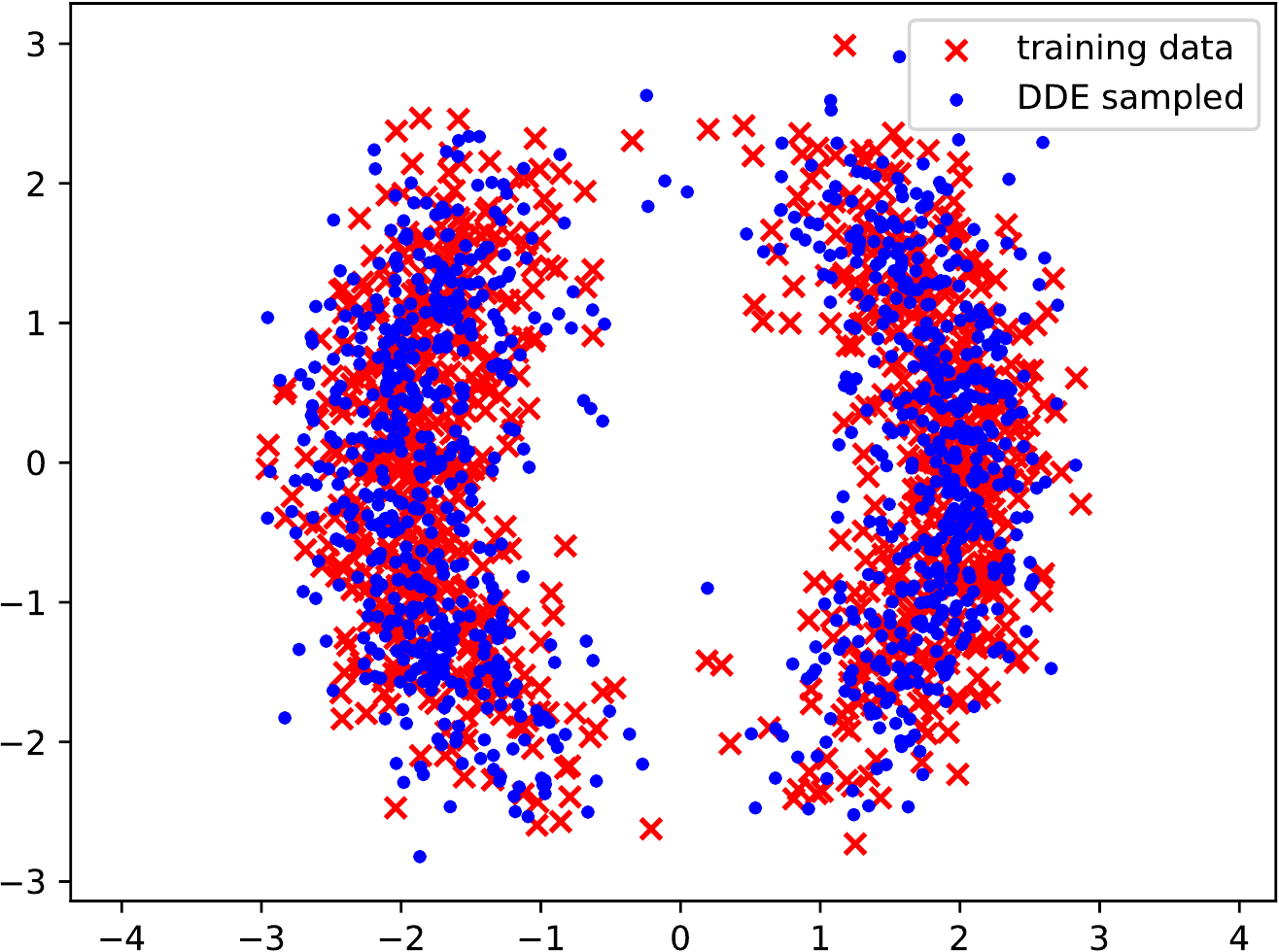}&
    \includegraphics[width=0.315\textwidth]{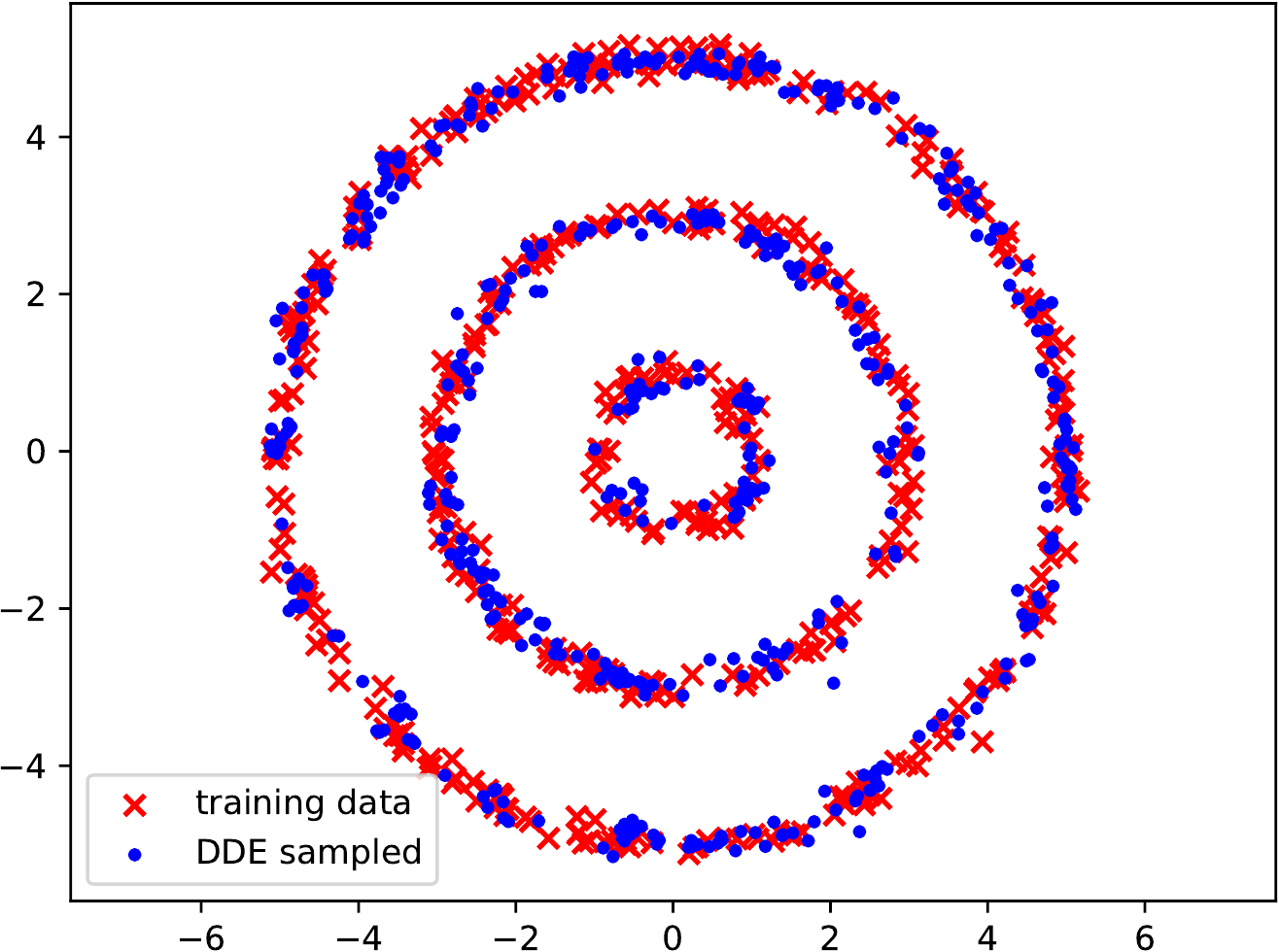}&
    \includegraphics[width=0.315\textwidth]{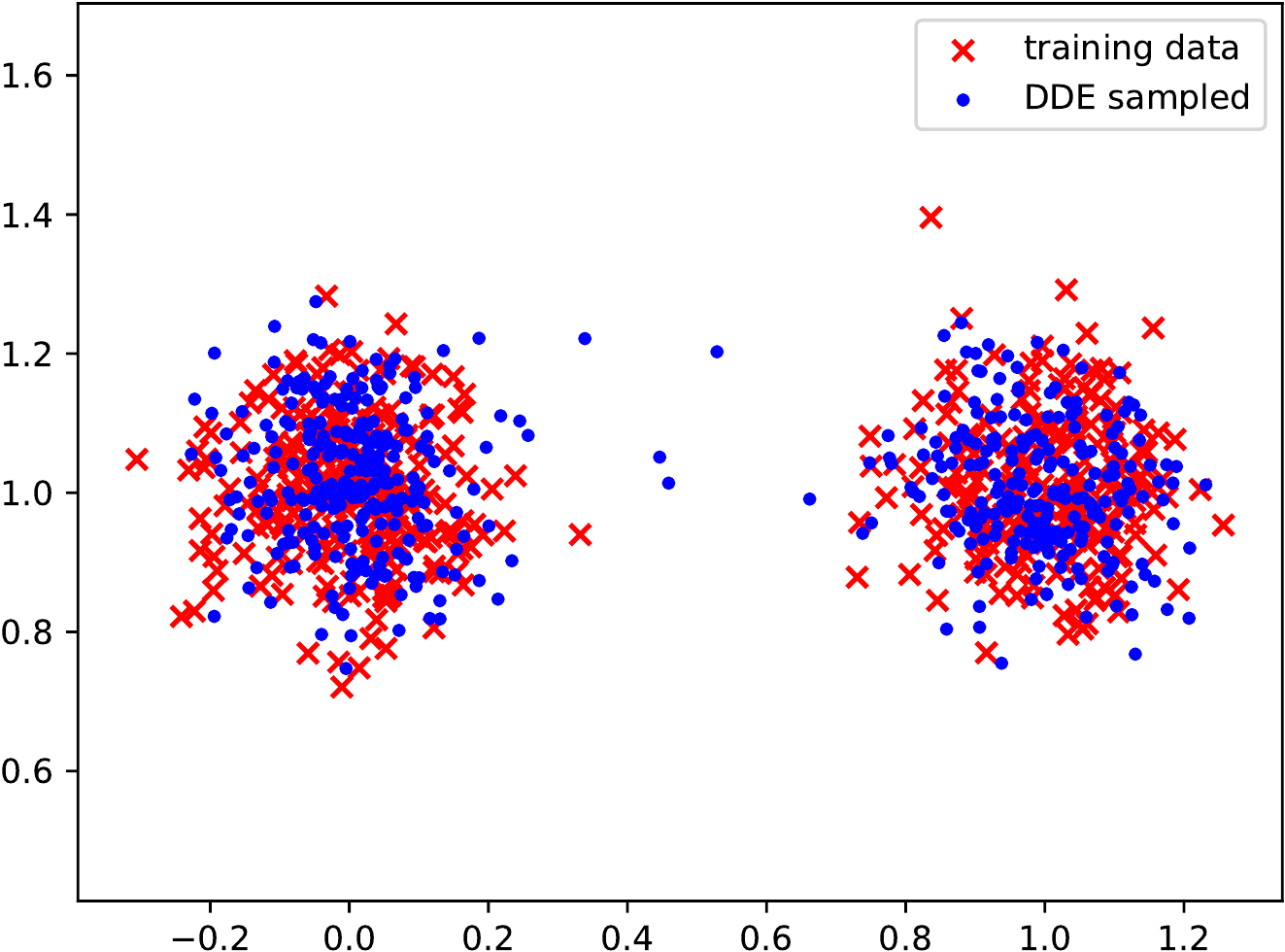}\\
    \includegraphics[width=0.3\textwidth]{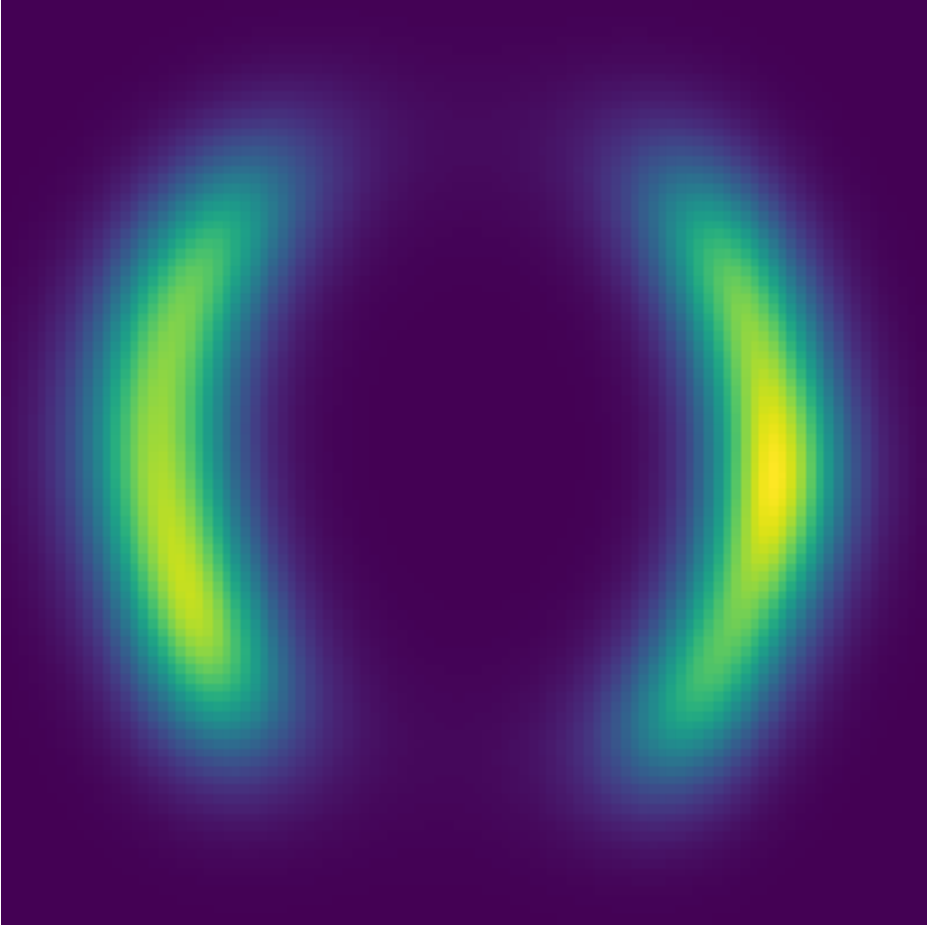}&
    \includegraphics[width=0.3\textwidth]{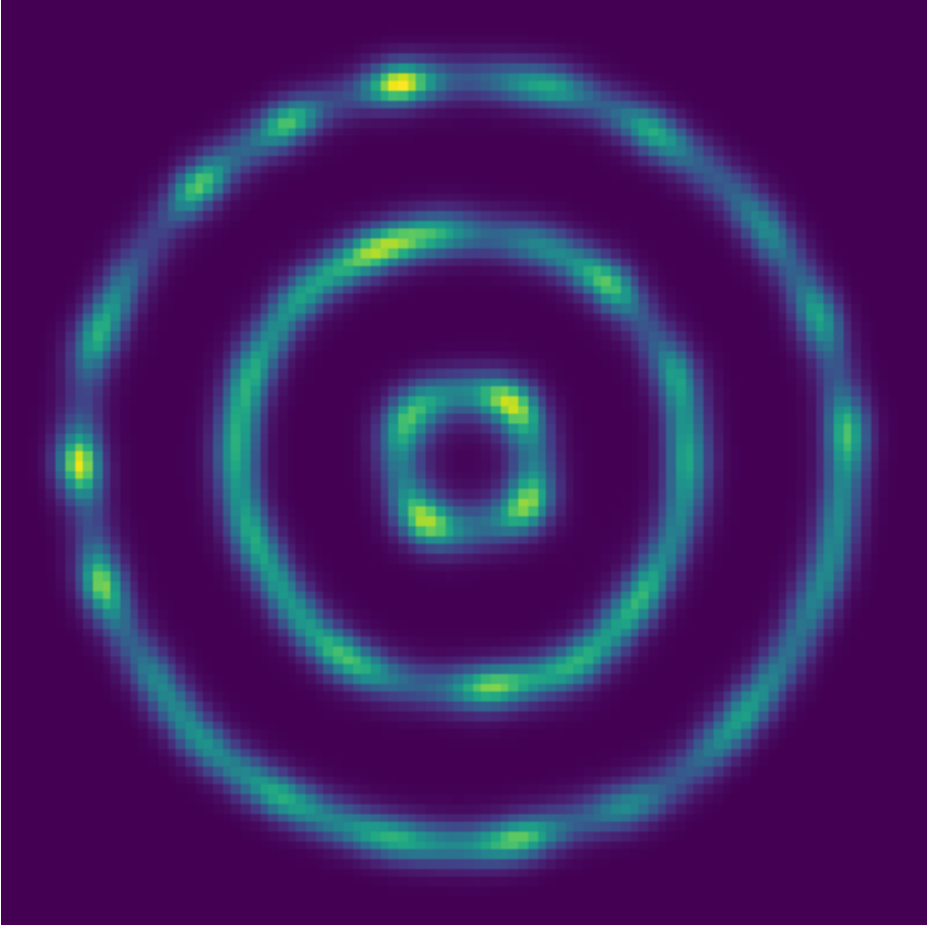}&
    \includegraphics[width=0.3\textwidth]{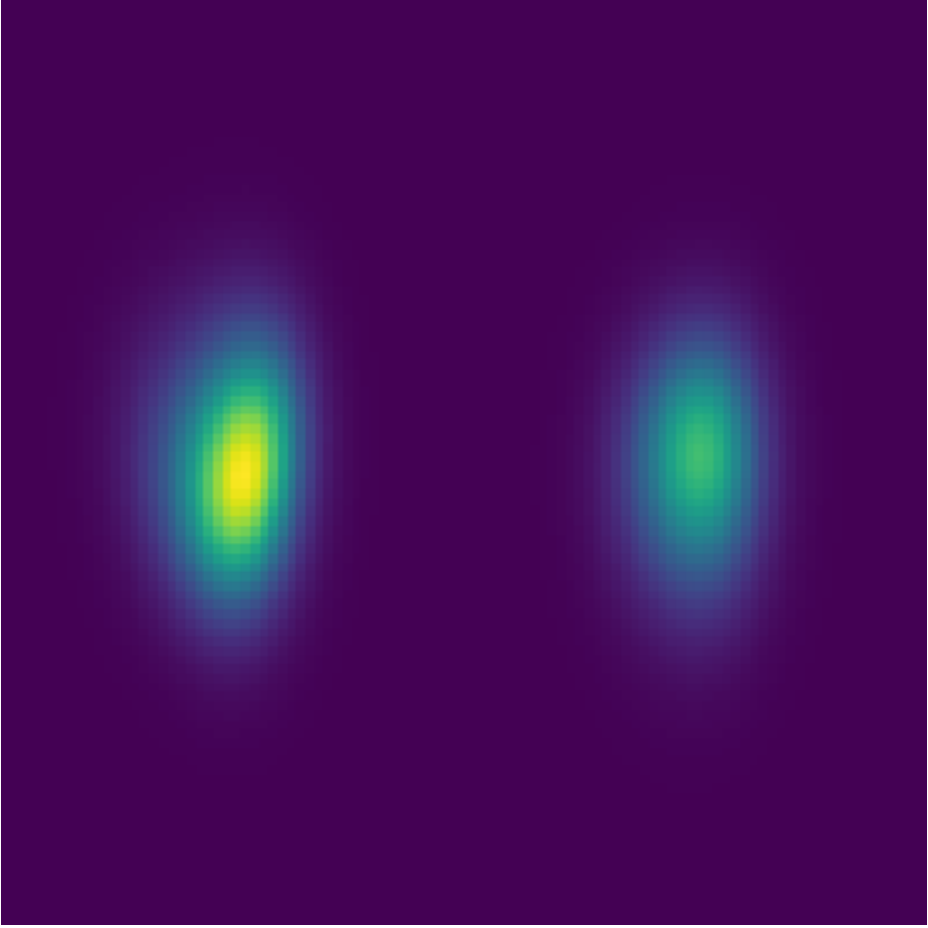}\\
  \end{tabular}
  \caption{The {\color{blue} blue} points in the figures in first row show the generated samples by learned models, and the {\color{red} red} points are the training samples. The learned $f$ are illustrated in the second row. }\label{fig:vis}
\end{figure}

In this section, we compare the proposed doubly dual embedding~(DDE) with the current state-of-the-art score matching estimators for kernel exponential family~(KEF)~\citet{SutStrArbArt17}\footnote{\url{https://github.com/karlnapf/nystrom-kexpfam}} and its conditional extension~(KCEF)~\citet{ArbGre17}\footnote{\url{https://github.com/MichaelArbel/KCEF}}, respectively, as well as several competitors. We test the proposed estimator empirically following their setting. We use Gaussian RBF kernel for both exponential family and its conditional extension, $k\rbr{x, x'} = \exp\rbr{-\nbr{x- x'}_2^2/\sigma^2}$, with the bandwidth $\sigma$ set by median-trick~\citep{DaiXieHeLiaEtAl14}. For a fair comparison, we follow~\citet{SutStrArbArt17} and~\citet{ArbGre17} to set the $p_0\rbr{x}$ for kernel exponential family and its conditional extension, respectively. The dual variables are parametrized by MLP with $5$ layers. More implementation details can be found in Appendix~\ref{appendix:impl_details} and the code repository which is available at \url{https://github.com/Hanjun-Dai/dde}.

\paragraph{Density estimation}
We evaluate the DDE on the synthetic datasets, including \texttt{ring}, \texttt{grid} and \texttt{two moons}, where the first two are used in~\citet{SutStrArbArt17}, and the last one is from~\citet{RezMoh15}. The \texttt{ring} dataset contains the points uniformly sampled along three circles with radii $\rbr{1,3,5}\in \RR^2$ and $\Ncal\rbr{0, 0.1^2}$ noise in the radial direction and extra dimensions. The $d$-dim \texttt{grid} dataset contains samples from mixture of $d$ Gaussians. Each center lies on one dimension in the $d$-dimension hypercube. The \texttt{two moons} dataset is sampled from the exponential family with potential function as $\frac{1}{2}\rbr{\frac{\nbr{x}-2}{0.4}}^2 - \log\rbr{\exp\rbr{-\frac{1}{2}\rbr{\frac{x-2}{0.6}}^2}+ \exp\rbr{-\frac{1}{2}\rbr{\frac{x+2}{0.6}}^2}}$. We use $500$ samples for training, and for testing $1500$ (\texttt{grid}) or $5000$ (\texttt{ring, two moons}) samples, following~\citet{SutStrArbArt17}.

We visualize the samples generated by the learned sampler, and compare it with the training datasets in the first row in~\figref{fig:vis}. The learned $f$ is also plotted in the second row in~\figref{fig:vis}. The DDE learned models generate samples that cover the training data, showing the ability of the DDE for estimating kernel exponential family on complicated data.  

Then, we demonstrate the convergence of the DDE in~\figref{fig:alg_conv} on \texttt{rings} dataset. We initialize with a random dual distribution. As the algorithm iterates, the distribution converges to the target true distribution, justifying the convergence guarantees. More results on $2$-dimensional \texttt{grid} and \texttt{two moons} can be found in~\figref{fig:alg_conv_more} in~\appref{appendix:more_exp}. The DDE algorithm behaves similarly on these two datasets.  
\begin{figure}[t]
\centering
  \begin{tabular}{ccc}
    \includegraphics[width=0.32\textwidth]{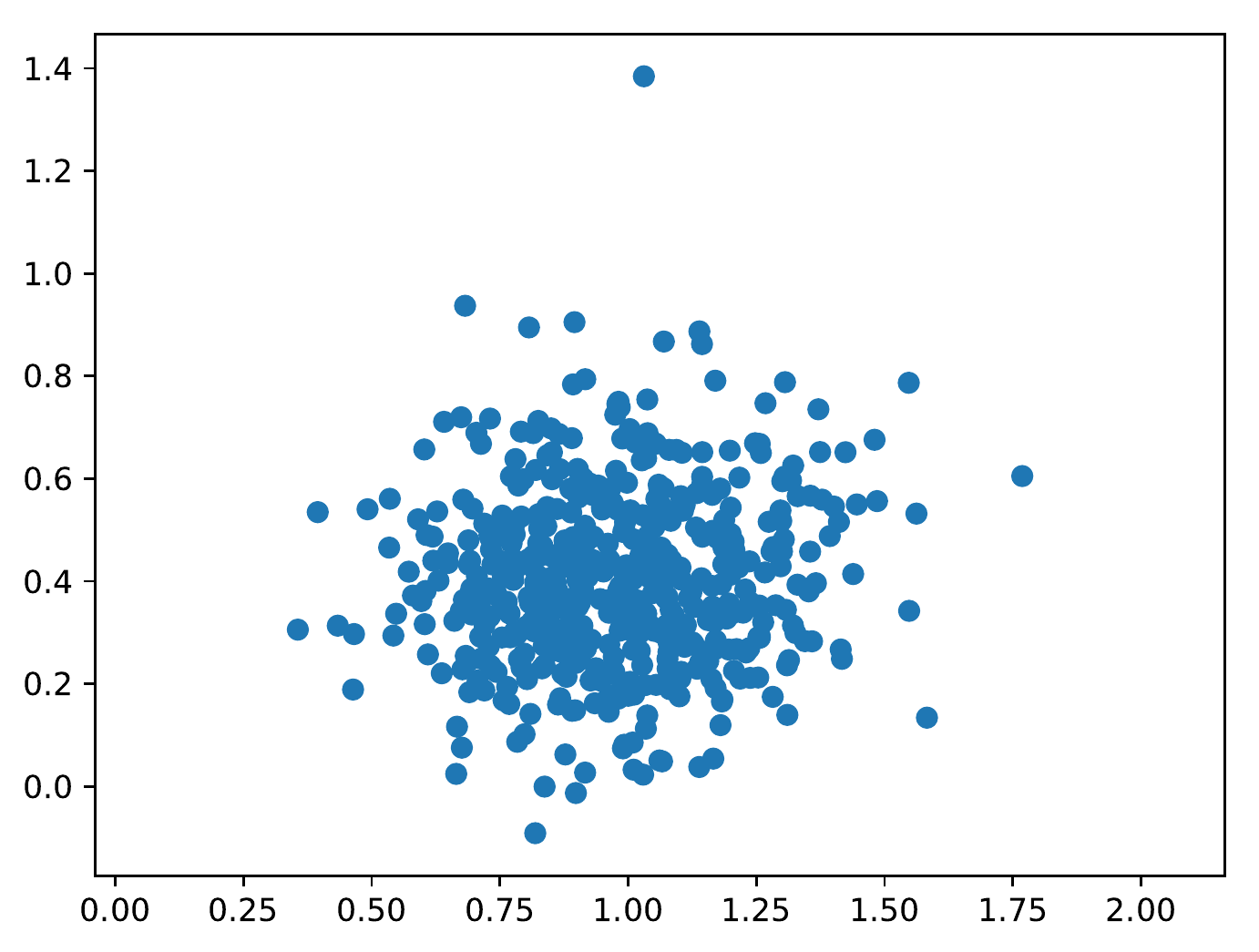}&
    \includegraphics[width=0.32\textwidth]{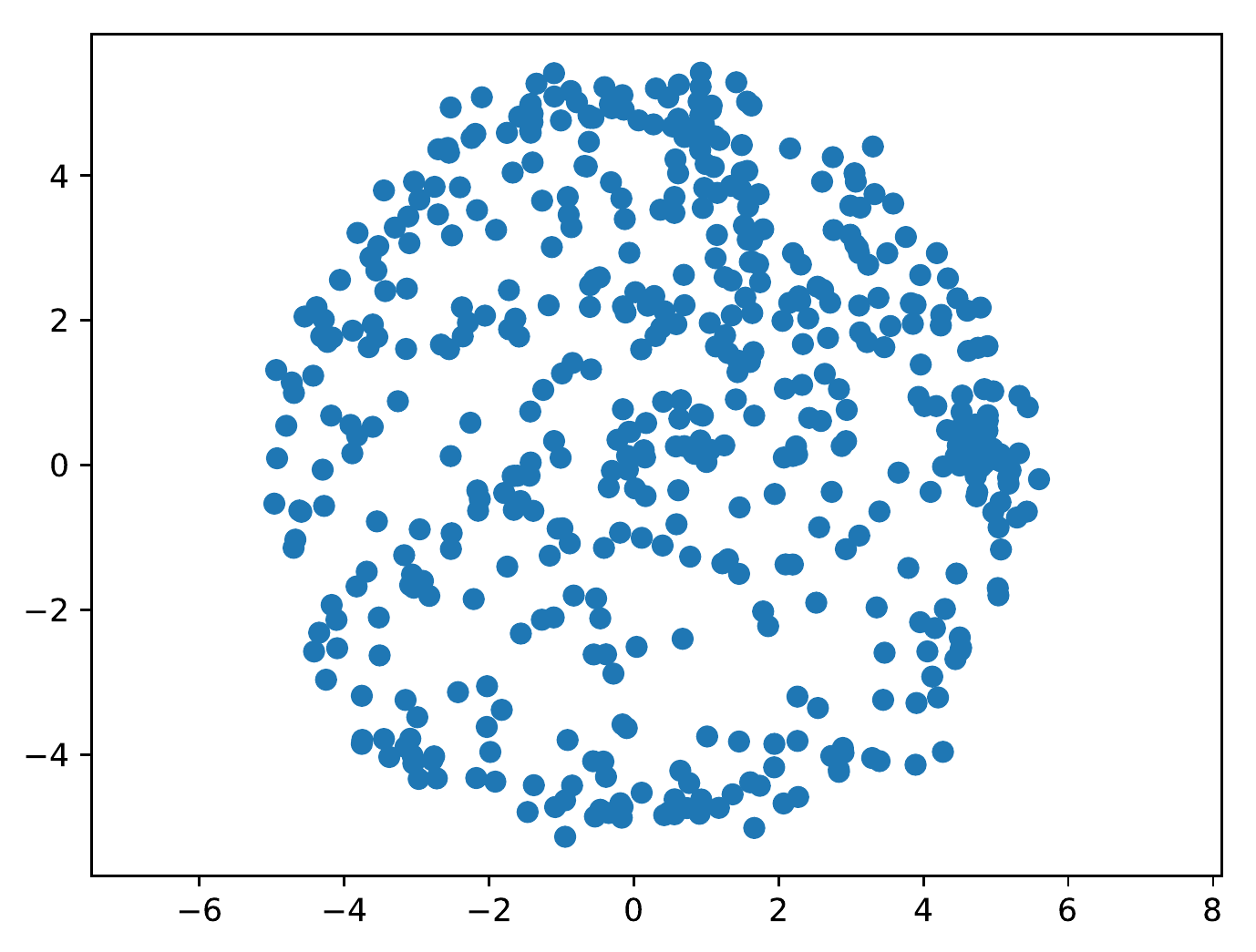}&
    \includegraphics[width=0.32\textwidth]{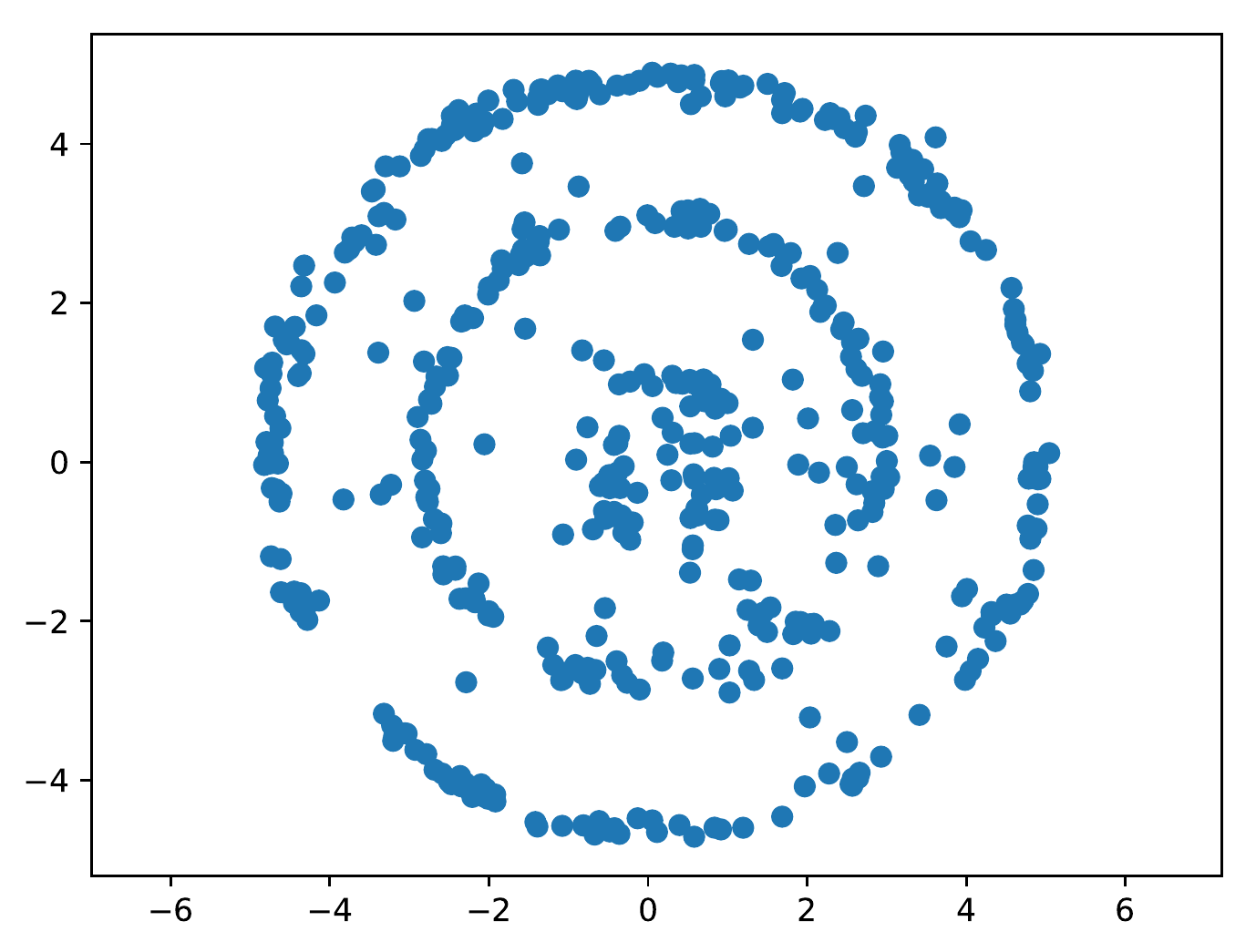}\\
    (a) initialization &(b) $500$-th   &(c) $2500$-th\\
  \end{tabular}
  \begin{tabular}{ccc}
    \includegraphics[width=0.32\textwidth]{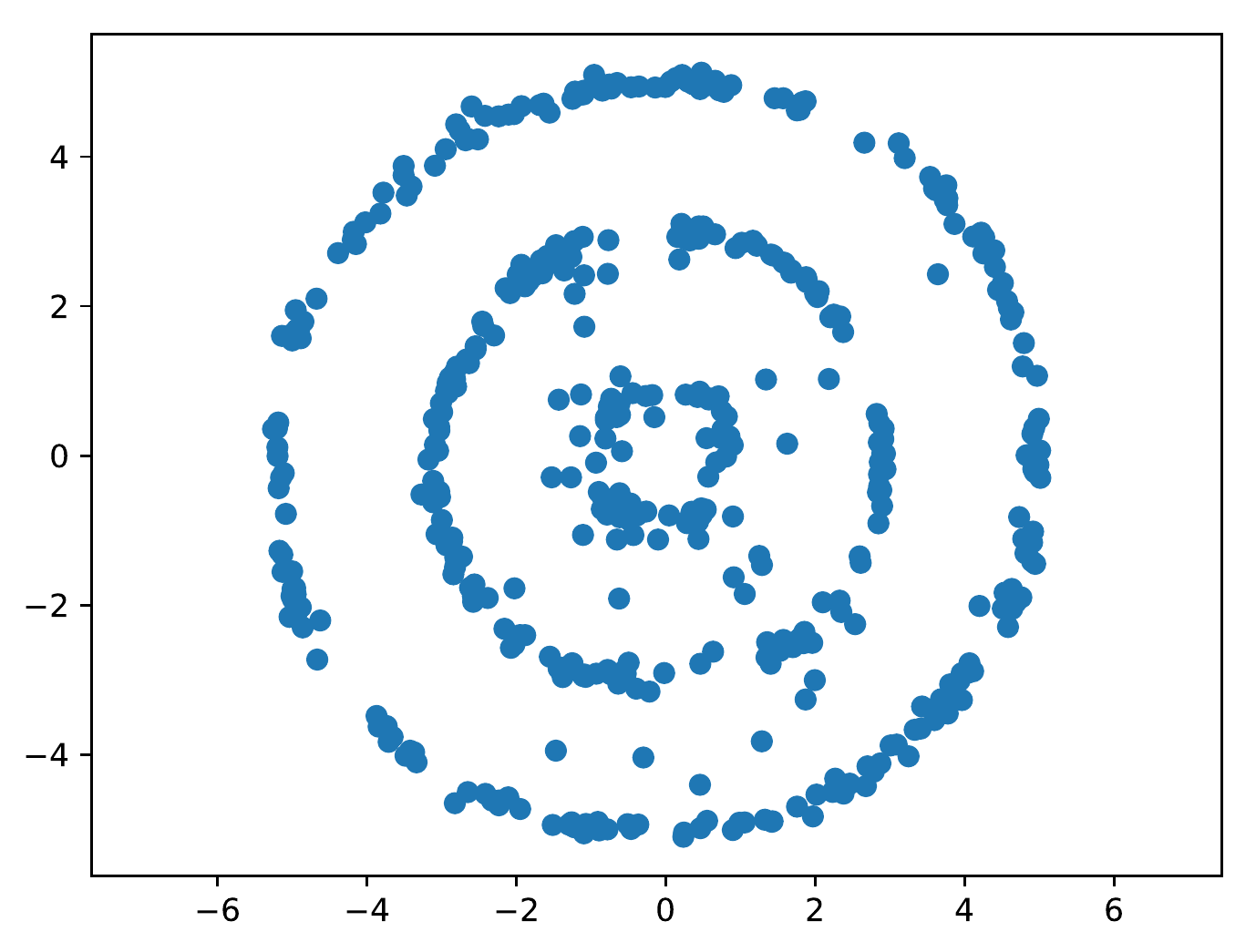}&
    \includegraphics[width=0.32\textwidth]{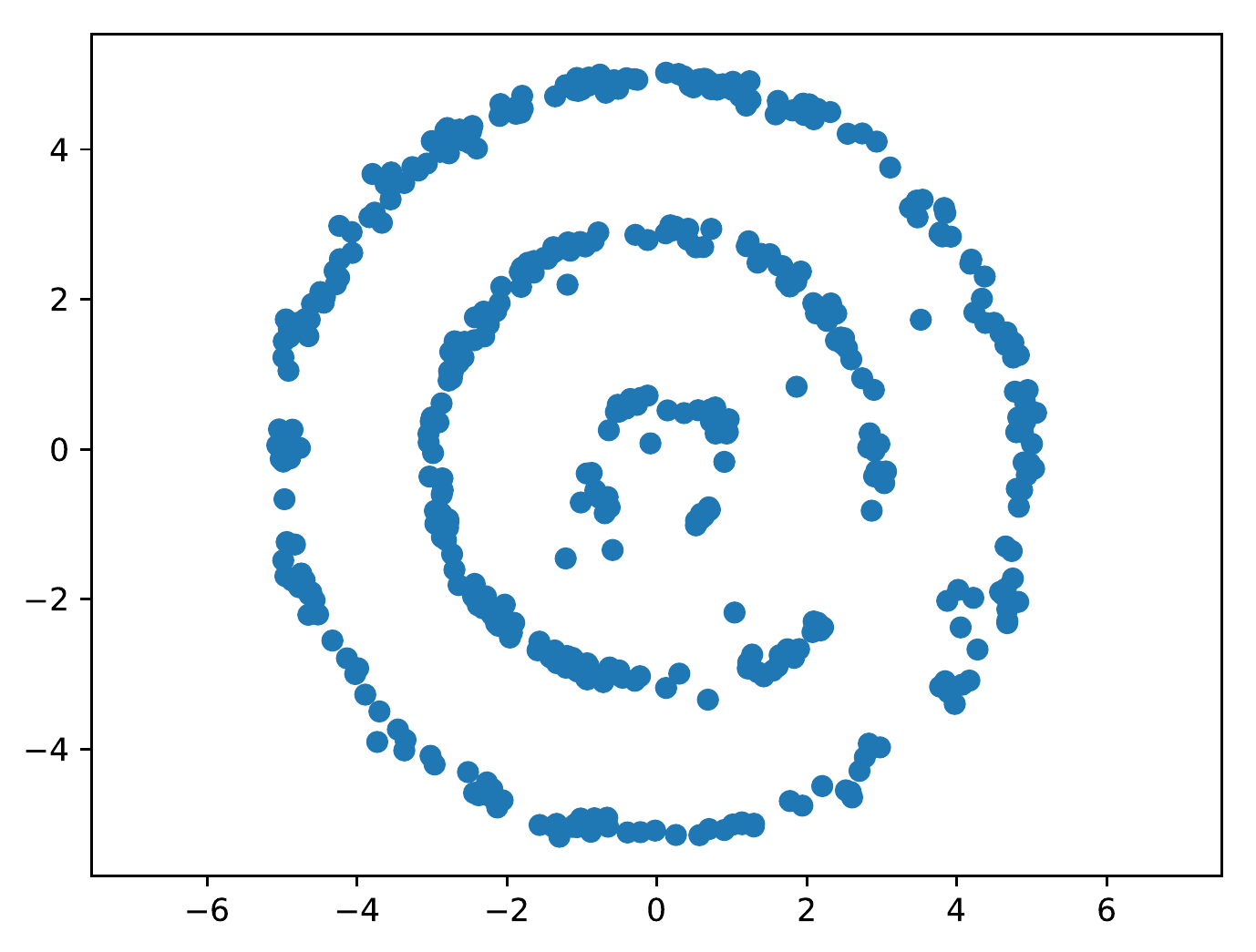}&
    \includegraphics[width=0.32\textwidth]{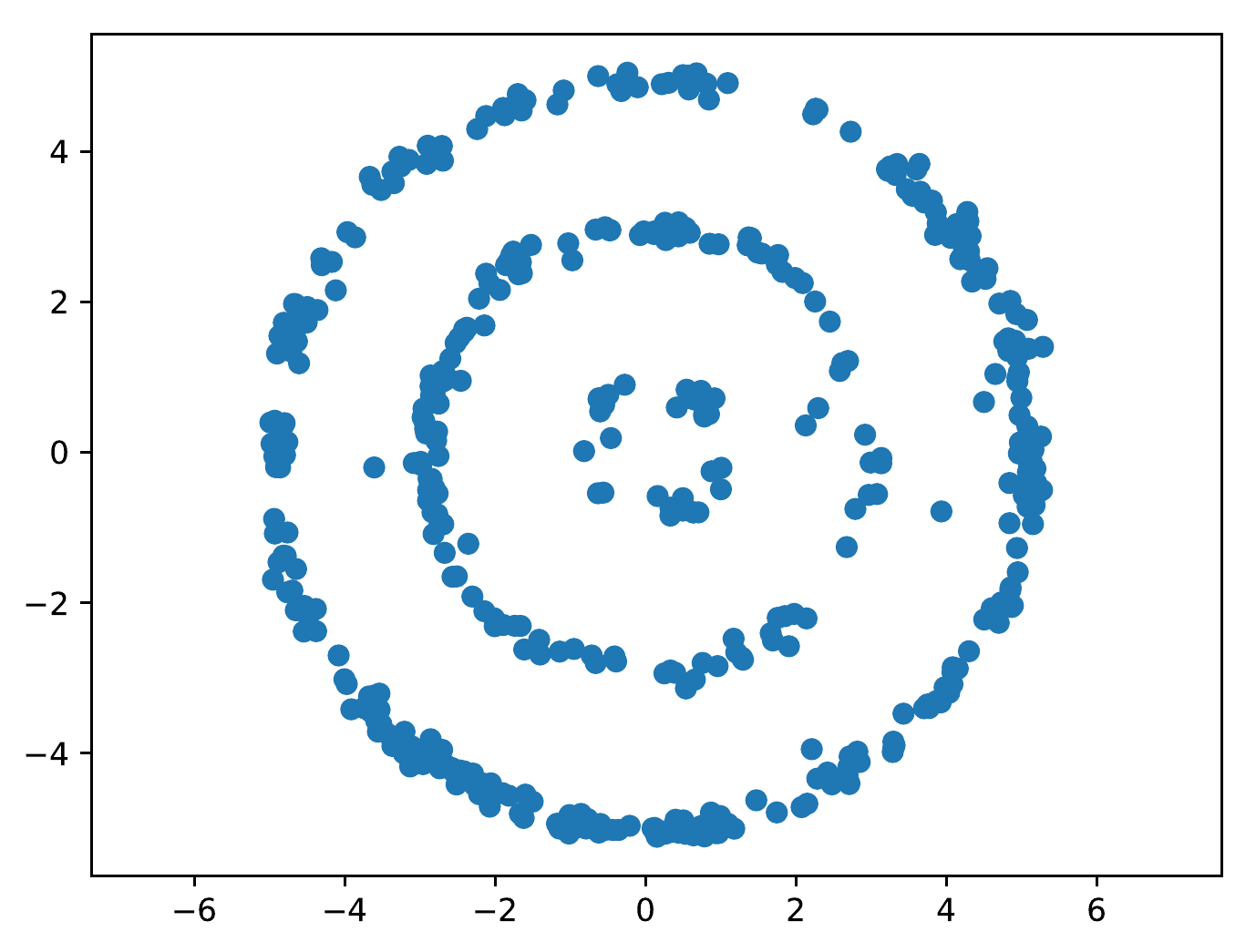}\\
    (d) $5000$-th  &(c) $10000$-th  &(d) $20000$-th \\
  \end{tabular}
  \caption{The DDE estimators on \texttt{rings} dataset in each iteration. The {\color{blue}blue} points are sampled from the learned model. With the algorithm proceeds, the learned distribution converges to the ground-truth. }
  \label{fig:alg_conv}
\end{figure}

\begin{table}[h]
\centering
\caption{Quantitative evaluation on synthetic data. MMD in $\times 10^{-3}$ scale on test set is reported.  \label{tab:mmd_synthetic}}
\begin{tabular}{ccccc}
 \toprule
  &\multicolumn{2}{c}{MMD} & \multicolumn{2}{c}{time (s)} \\
 \cmidrule(l){2-3} \cmidrule(l){4-5} 
 Datasets & DDE & KEF & DDE-sample & KEF+HMC  \\
 \cmidrule(l){1-1} \cmidrule(l){2-3} \cmidrule(l){4-5} 
 two moons & \textbf{0.11}  & 1.32 & {\bf 0.07} & 13.04\\
 ring & \textbf{1.53} & 3.19 & {\bf 0.07} & 14.12 \\
 grid & \textbf{1.32}  & 40.39 & {\bf 0.04} & 4.43 \\
 \bottomrule
\end{tabular}
\end{table}

Finally, we compare the MMD between the generated samples from the learned model to the training data with the current state-of-the-art method for KEF~\citep{SutStrArbArt17}, which performs better than KDE and other alternatives~\citep{StrSejLivSzaetal15}. We use HMC to generate the samples from the KEF learned model, while in the DDE, we can bypass the HMC step by the learned sampler. The computation time for inference is listed in~\tabref{tab:mmd_synthetic}. It shows the DDE improves the inference efficiency in orders. The MMD comparison is listed in~\tabref{tab:mmd_synthetic}. The proposed DDE estimator performs significantly better than the best performances by KEF in terms of MMD.

\paragraph{Conditional model extension.} 

In this part of the experiment, models are trained to estimate the conditional distribution $p\rbr{y|x}$ on the benchmark datasets for studying these methods~\citep{ArbGre17,SugTakSuzKanetal10}. We centered and normalized the data and randomly split the datasets into a training and a testing set with equal size, as in~\citet{ArbGre17}. We evaluate the performances by the negative $log$-likelihood. Besides the score matching based KCEF, we also compared with LS-CDE and $\epsilon$-KDE, introduced in~\citep{SugTakSuzKanetal10}. The empirical results are summarized in~\tabref{tab:r_benchmark}.

\begin{table}[h]
\centering
\caption{The negative log-likelihood comparison on benchmarks. Mean and std are calculated on 20 runs of different train/test splits. KCEF gets numerically unstable on two datasets, where we mark as N/A.
}\label{tab:r_benchmark}%
\begin{tabular}{ccccc}
 \toprule
 & DDE & KCEF & $\epsilon$-KDE & LSCDE  \\
 \midrule
 geyser & {\bf 0.55 $\pm$ 0.07} & 1.21 $\pm$ 0.04 & 1.11 $\pm$ 0.02 & 0.7 $\pm$ 0.01 \\
 caution & {\bf 0.95 $\pm$ 0.19} & {\bf 0.99 $\pm$ 0.01} & 1.25 $\pm$ 0.19 & 1.19 $\pm$ 0.02\\
 ftcollinssnow & {\bf 1.49 $\pm$ 0.14} & {\bf 1.46 $\pm$  0.0} & 1.53 $\pm$ 0.05 & 1.56 $\pm$ 0.01 \\
 highway &  {\bf 1.18 $\pm$ 0.30} & {\bf 1.17 $\pm$ 0.01} & 2.24 $\pm$ 0.64  & 1.98 $\pm$ 0.04 \\
 snowgeese & {\bf 0.42 $\pm$ 0.31} & {\bf 0.72 $\pm$ 0.02} & 1.35 $\pm$ 0.17 & 1.39 $\pm$ 0.05 \\
 GAGurine & {\bf 0.43 $\pm$ 0.15} & {\bf 0.46 $\pm$ 0.0} & 0.92 $\pm$ 0.05 & 0.7 $\pm$ 0.01 \\
 topo & 1.02 $\pm$ 0.31 & {\bf 0.67 $\pm$ 0.01} & 1.18 $\pm$ 0.09 & 0.83 $\pm$ 0.0 \\
 CobarOre & {\bf 1.45 $\pm$ 0.23} & 3.42 $\pm$ 0.03 & {\bf 1.65 $\pm$ 0.09} & {\bf 1.61 $\pm$ 0.02} \\
 mcycle & {\bf 0.60 $\pm$ 0.15} &  {\bf 0.56 $\pm$ 0.01} & 1.25 $\pm$ 0.23 & 0.93 $\pm$ 0.01 \\
 BigMac2003 & {\bf 0.47 $\pm$ 0.36} & {\bf 0.59 $\pm$  0.01} & 1.29 $\pm$  0.14 & 1.63 $\pm$  0.03 \\
 cpus & {\bf -0.63 $\pm$ 0.77} & N/A & 1.01 $\pm$ 0.10 & 1.04 $\pm$ 0.07 \\
 crabs & {\bf -0.60 $\pm$ 0.26} & N/A & 0.99 $\pm$ 0.09 & -0.07 $\pm$ 0.11 \\
 birthwt & {\bf 1.22 $\pm$ 0.15} & {\bf 1.18 $\pm$ 0.13} & 1.48 $\pm$ 0.01 & 1.43 $\pm$ 0.01 \\
 gilgais & {\bf 0.61 $\pm$ 0.10} & {\bf 0.65 $\pm$ 0.08} & 1.35 $\pm$ 0.03 & 0.73 $\pm$ 0.05\\
 UN3 & {\bf 1.03 $\pm$ 0.09} & 1.15 $\pm$ 0.21 & 1.78 $\pm$ 0.14 & 1.42 $\pm$ 0.12  \\
 ufc & {\bf 1.03 $\pm$ 0.10} & {\bf 0.96 $\pm$ 0.14} & 1.40 $\pm$ 0.02 & {\bf 1.03 $\pm$ 0.01}\\
 \bottomrule
\end{tabular}
\end{table}

Although these datasets are low-dimensional with few samples and the KCEF uses the anisotropic RBF kernel~(\ie, different bandwidth in each dimension, making the experiments preferable to the KCEF), the proposed DDE still outperforms the competitors on six datasets significantly, and achieves comparable performance on the rest, even though it uses a simple isotropic RBF kernel. This further demonstrates the statistical power of the proposed DDE, comparing to the score matching estimator.

%%%%%%%%%%%%%%%%%%%%%%%%%%%%%%%%%%%%%%%%%%%%%%%%%%%%%%%%%%%%%%%%%%%%%%%%%%%%%%%%%%%%%%%%%%%%
\section{Conclusion}\label{sec:conclusion}
%%%%%%%%%%%%%%%%%%%%%%%%%%%%%%%%%%%%%%%%%%%%%%%%%%%%%%%%%%%%%%%%%%%%%%%%%%%%%%%%%%%%%%%%%%%%
In this paper, we exploit the \emph{doubly dual embedding} to reformulate the penalized MLE to a novel saddle-point optimization, which bypasses the intractable integration and provides flexibility in parameterizing the dual distribution. The saddle point view reveals a unique understanding of GANs and leads to a practical algorithm, which achieves state-of-the-art performance. We also establish the statistical consistency and algorithm convergence guarantee for the proposed algorithm. Although the transport mapping parametrization is flexible enough, it requires extra optimization for the $KL$-divergence estimation. For the future work, we will exploit dynamic-based sampling methods to design new parametrization, which shares both flexibility and density tractability.

%%%%%%%%%%%%%%%%%%%%%%%%%%%%%%%%%%%%%%%%%%%%%%%%%%%%%%%%%%%%
\subsubsection*{Acknowledgements}
%%%%%%%%%%%%%%%%%%%%%%%%%%%%%%%%%%%%%%%%%%%%%%%%%%%%%%%%%%%%
We thank Michael Arbel and the anonymous reviewers of AISTASTS 2019 for their insightful comments and suggestions. NH is supported in part by NSF-CRII-1755829, NSF-CMMI-1761699, and NCSA Faculty Fellowship. 

%%%%%%%%%%%%%%%%%%%%%%%%%%%%%%%%%%%%%%%%%%%%%%%%%%%%%%%%%%%%
%%%% Reference
%%%%%%%%%%%%%%%%%%%%%%%%%%%%%%%%%%%%%%%%%%%%%%%%%%%%%%%%%%%%

%----------------------------------------------------------------------------------------------------------------------------------
%----------------------------------------------------------------------------------------------------------------------------------
\clearpage
\newpage

\appendix

\begin{appendix}

\begin{center}
{\huge Appendix}
\end{center}

%%%%%%%%%%%%%%%%%%%%%%%%%%%%%%%%%%%%%%%%%%%%%%%%%%%%%%%%%%%%%%%%%%%%%%%%%%%%%%%%%%%%%%%%%%%%
\section{Proof Details}\label{appendix:proof_details}
%%%%%%%%%%%%%%%%%%%%%%%%%%%%%%%%%%%%%%%%%%%%%%%%%%%%%%%%%%%%%%%%%%%%%%%%%%%%%%%%%%%%%%%%%%%%

%%%%%%%%%%--------------------------------------------------------------------------------
\subsection{Proof of~\thmref{thm:dual_gradient}}\label{appendix:subsec:dual_grad}
%%%%%%%%%%--------------------------------------------------------------------------------

{\bf\thmref{thm:dual_gradient} (Dual gradient)}
\emph{Denoted as $\rbr{f^*, \nu^*} = \argmax_{\rbr{f, \nu}\in\Hcal} \tilde \ell\rbr{f, \nu, w_g}$ and $\Lhat\rbr{w_g} = \tilde\ell\rbr{f^*, \nu^*, w_g}$, we have 
\begin{eqnarray*}
\nabla_{w_g} \Lhat\rbr{w_g} = -\EE_{\xi}\sbr{\nabla_{w_g} f^*\rbr{g_{w_g}\rbr{\xi}}} + \frac{1}{\lambda}\EE_{\xi}\sbr{\nabla_{w_g} \nu^*\rbr{g_{w_g}\rbr{\xi}}}.\nonumber
\end{eqnarray*}
}

\begin{proof}
The conclusion can be proved by chain rule and the optimality conditions. 

Specifically, notice that the $\rbr{f_{w_g}^*, \nu_{w_g}^*}$ are implicit functions of $w_g$, we can calculate the gradient of $\Lhat\rbr{w_g}$ w.r.t. $w_g$
\begin{eqnarray*}
\nabla_{w_g}\Lhat\rbr{w_g} &=& \widehat\EE_{\Dcal}\sbr{\nabla_f f_{w_g}^*\nabla_{w_g}f^*_{w_g}} - \EE_\xi\sbr{\nabla_{g}f\rbr{g\rbr{\xi}}\nabla_{w_g} g}- \EE_q\sbr{\nabla_f f_{w_g}^*\nabla_{w_g}f^*_{w_g}} -\frac{\eta}{2}\nabla_{f}\nbr{f_{w_g}^*}_{\Hcal}^2\nabla_{w_g}f_{w_g}^*\\
&& +\frac{1}{\lambda }\rbr{\EE_{\xi}\sbr{\nabla_g\nu_{w_g}^*\rbr{g\rbr{\xi}}\nabla_{w_g}g} + \EE_q\sbr{\nabla_{\nu}\nu_{w_g}^*\nabla_{w_g}\nu^*_{w_g}} - \EE_{p_0}\sbr{\exp\rbr{\nu_{w_g}^*}\nabla_{\nu}\nu_{w_g}^*\nabla_{w_g}\nu_{w_g}^*}}\\
&=& \underbrace{\rbr{\widehat\EE_{\Dcal}\sbr{\nabla_f f_{w_g}^*} - \EE_q\sbr{\nabla_f f_{w_g}^*} -\frac{\eta}{2}\nabla_{f}\nbr{f_{w_g}^*}_{\Hcal}^2}}_{0}\nabla_{w_g}f_{w_g}^*- \EE_\xi\sbr{\nabla_{g}f\rbr{g\rbr{\xi}}\nabla_{w_g} g}\\
&& +\frac{1}{\lambda }\EE_{\xi}\sbr{\nabla_g\nu_{w_g}^*\rbr{g\rbr{\xi}}\nabla_{w_g}g} + \frac{1}{\lambda}\underbrace{\rbr{\EE_q\sbr{\nabla_{\nu}\nu_{w_g}^*} - \EE_{p_0}\sbr{\exp\rbr{\nu_{w_g}^*}\nabla_{\nu}\nu_{w_g}^*}}}_{0}\nabla_{w_g}\nu_{w_g}^*\\
&=&-\EE_{\xi}\sbr{\nabla_{w_g} f^*\rbr{g\rbr{\xi}}} + \frac{1}{\lambda}\EE_{\xi}\sbr{\nabla_{w_g} \nu^*\rbr{g_{w_g}\rbr{\xi}}},
\end{eqnarray*} 
where the second equations come from the fact $\rbr{f_{w_g}^*, \nu_{w_g}^*}$ are optimal and $\rbr{\nabla_{w_g}f_{w_g}^*, \nabla_{w_g}\nu_{w_g}^*}$ are not functions of $\rbr{x, \xi, x'}$. 

\end{proof}

%%%%%%%%%%--------------------------------------------------------------------------------
\subsection{Proof of~\thmref{thm:consistency}}\label{appendix:subsec:consistency}
%%%%%%%%%%--------------------------------------------------------------------------------
The proof of~\thmref{thm:consistency} mainly follows the technique in~\citet{GuQiu93} with extra consideration of the approximation error from the dual embedding. 

We first define some notations that will be used in the proof. We denote $\langle f, g \rangle_{p} = \int_\Omega f\rbr{x}g\rbr{x}p\rbr{x}dx$, which induces the norm denoted as $\nbr{\cdot}_p^2$. We introduce $\htil$ as the maximizer to $\Ltil\rbr{h}$ defined as
$$
\Ltil\rbr{h} \defeq \widehat\EE_\Dcal\sbr{h} - \EE_{p^*}\sbr{h} - \frac{1}{2}\nbr{h - f^*}_{p^*}^2 - \frac{\eta}{2}\nbr{h}_\Hcal^2.
$$
The proof relies on decomposing the error into two parts: {\bf i)} the error between $\htil$ and $f^*$; and {\bf ii)} the error between $\ftil$ and $\htil$.

By Mercer decomposition~\citep{Konig86}, we can expand $k\rbr{\cdot, \cdot}$ as 
$$
k\rbr{x, x'} = \sum_{l=1}^\infty \zeta_l \psi_l\rbr{x}\psi_l\rbr{x'}, 
$$
With the eigen-decomposition, we can rewrite function $f\in \Hcal$ as $f\rbr{\cdot} = \sum_{l=1}^\infty \langle f, \psi_l \rangle_{p^*} \psi\rbr{\cdot}$. Then, we have
$\nbr{f}_{\Hcal}^2 = \sum_{l=1}^\infty \zeta_l^{-1}\langle f, \psi_l \rangle_{p^*}^2$ and $\nbr{f}_{p^*}^2 = \sum_{l=1}^\infty \langle f, \psi_l \rangle_{p^*}^2$. 

We make the following standard assumptions:
\begin{assumption}\label{asmp:kernel_property}
There exists $\kappa >0$ such that $k\rbr{x, x'}\le \kappa$, $\forall x, x'\in\Omega$. 
\end{assumption}
\begin{assumption}\label{asmp:exp_decay}
The eigenvalues of the kernel $k\rbr{\cdot, \cdot}$ decay sufficiently homogeneously with rate $r$, \ie, $\zeta_l = \Ocal\rbr{l^{-r}}$ where $r>1$. 
\end{assumption}
\begin{assumption}\label{asmp:prior}
There exists a distribution $p_0$ on the support $\Omega$ which is uniformly upper and lower bounded. 
\end{assumption}

\begin{lemma}\label{lemma:bounded_pf}
Under~\asmpref{asmp:prior}, $\forall f\in\Hcal_k$, $p_{f}\rbr{x} = \frac{\exp\rbr{f\rbr{x} - \log p_0\rbr{x}}}{\int_\Omega \exp\rbr{f\rbr{x} - \log p_0\rbr{x}} p_0\rbr{x}}p_0\rbr{x}$, we have $2\exp\rbr{-\kappa C_{\Hcal} - C_0}\le p^*\rbr{x}\le 2\exp\rbr{\kappa C_{\Hcal} + C_0}$.
\end{lemma}
\begin{proof}
For $\forall f\in \Hcal_k$ and the~\asmpref{asmp:prior}, $p_{f}\rbr{x} = \frac{\exp\rbr{f\rbr{x} - \log p_0\rbr{x}}}{\int_\Omega \exp\rbr{f\rbr{x} - \log p_0\rbr{x}} p_0\rbr{x}}p_0\rbr{x} $ with $f\in\Hcal_k$ and $\nbr{f}_{\Hcal}\le C_{\Hcal}$ and $\nbr{\log p_0\rbr{x}}_\infty \le C_0$, implies $2\exp\rbr{-\kappa C_{\Hcal} - C_0}\le p^*\rbr{x}\le 2\exp\rbr{\kappa C_{\Hcal} + C_0}$. 
\end{proof}

Therefore, it is reasonable to consider the parametrization of the dual distribution:
\begin{assumption}\label{asmp:q_param}
The parametric family of dual distribution in~\eqref{eq:double_dual_mle} is bounded above from infinity, \ie, $\forall q\rbr{x}\in\Pcal_w$, $q\rbr{x}\le C_{\Pcal_w}<\infty$, $\forall x\in \Omega$. 
\end{assumption}

To prove the~\thmref{thm:consistency}, we first show the error between $\htil$ and $f^*$ under~\asmpref{asmp:kernel_property}, \ref{asmp:exp_decay},~\ref{asmp:prior}, and~\ref{asmp:q_param}. 
\begin{lemma}\label{lemma:intermediate_I}
Under~\asmpref{asmp:exp_decay}, we have 
\begin{eqnarray*}
\EE\sbr{\nbr{\htil - f^*}^2_{p^*}} = \Ocal\rbr{n^{-1}\eta^{-\frac{1}{r}} + \eta},\\
\eta\EE\sbr{\nbr{\htil - f^*}^2_{\Hcal}} = \Ocal\rbr{n^{-1}\eta^{-\frac{1}{r}} + \eta}. 
\end{eqnarray*} 
\end{lemma}
\begin{proof}
Denote the $\htil\rbr{\cdot} = \sum_{l}^\infty \underbrace{\langle \htil, \psi_l \rangle_{p^*}}_{\htil_l}\psi_l\rbr{\cdot}$ and $f^*\rbr{\cdot} = \sum_{l}^\infty \underbrace{\langle f^*, \psi_l \rangle_{p^*}}_{f^*_l}\psi_l\rbr{\cdot}$, then, we can rewrite the $\Ltil\rbr{h}$ as 
$$
\Ltil\rbr{h} = \sum_{l=1}^\infty h_l\sbr{\widehat\EE\sbr{\psi_l\rbr{x}} -\EE_{p^*}\sbr{\psi_l\rbr{x}}} -\frac{1}{2}\sum_{l=1}^\infty \rbr{h_l - f^*_l}^2 - \frac{\eta}{2}\sum_{l=1}^\infty\zeta_l^{-1}h_l^2.
$$
Setting the derivative of $\Ltil\rbr{h}$ w.r.t. $\sbr{h_l}$ equal to zero, we obtain the representation of $\htil_l$ as
$$
\htil_l = \frac{f^*_l + \alpha_l}{1 + \eta\zeta_l^{-1}},
$$
where $\alpha_l = \widehat\EE\sbr{\psi_l\rbr{x}} - \EE_{p^*}\sbr{\psi_l\rbr{x}}$. Then, we have 
\begin{eqnarray*}
\nbr{\htil - f^*}_{p^*}^2 = \sum_{l=1}^\infty \rbr{\htil_l - f^*_l}^2 = \sum_{l=1}^\infty \frac{\alpha_l^2 - 2\alpha_l\eta\zeta_l^{-1}f^*_l + \eta^2\zeta_l^{-2}\rbr{f_l^*}^2}{\rbr{1 + \eta\zeta_l^{-1}}^2},\\ 
\eta\nbr{\htil - f^*}_{\Hcal}^2 = \eta\sum_{l=1}^\infty \zeta_l^{-1}\rbr{\htil_l - f^*_l}^2 = \sum_{l=1}^\infty \eta\zeta_l^{-1}\frac{\alpha_l^2 - 2\alpha_l\eta\zeta_l^{-1}f^*_l + \eta^2\zeta_l^{-2}\rbr{f_l^*}^2}{\rbr{1 + \eta\zeta_l^{-1}}^2}. 
\end{eqnarray*}
Recall that $\EE\sbr{a_l} = 0$ and $\EE\sbr{a_l^2}=\frac{1}{n}$, then we have
\begin{eqnarray}\label{eq:expectation_expression}
\EE\sbr{\nbr{\htil - f^*}_{p^*}^2} = \frac{1}{n}\sum_{l=1}^\infty \frac{1}{\rbr{1 + \eta\zeta_l^{-1}}^2} + \eta\sum_{l=1}^\infty \frac{\eta\zeta_l^{-1}}{\rbr{1 + \eta\zeta_l^{-1}}^2}\cdot\zeta_t^{-1}\rbr{f_l^*}^2,\\ 
\EE\sbr{\eta\nbr{\htil - f^*}_{\Hcal}^2} = \frac{1}{n}\sum_{l=1}^\infty \frac{\eta\zeta_l^{-1}}{\rbr{1 + \eta\zeta_l^{-1}}^2} + \eta\sum_{l=1}^\infty\frac{\eta^2\zeta_l^{-2}}{\rbr{1 + \eta\zeta_l^{-1}}^2}\cdot\zeta_l^{-1}\rbr{f_l^*}^2. 
\end{eqnarray}
By calculation, we obtain that
\begin{eqnarray}\label{eq:order_I}
\sum_{l=1}^\infty \frac{\eta\zeta_l^{-1}}{\rbr{1 + \eta\zeta_l^{-1}}^2} &=& \sum_{l<\eta^{-\frac{1}{r}}} \frac{\eta\zeta_l^{-1}}{\rbr{1 + \eta\zeta_l^{-1}}^2} + \sum_{l\ge \eta^{-\frac{1}{r}}} \frac{\eta\zeta_l^{-1}}{\rbr{1 + \eta\zeta_l^{-1}}^2} \nonumber\\
&=& \Ocal\rbr{\eta^{-\frac{1}{r}}} +\Ocal\rbr{\int_{\eta^{-\frac{1}{r}}}^\infty \frac{\eta t^r}{\rbr{1+\eta t^r}^2}dt} \nonumber\\
& = & \Ocal\rbr{\eta^{-\frac{1}{r}}} + \eta^{-\frac{1}{r}}\Ocal\rbr{\int_{1}^\infty \frac{ t^r}{\rbr{1+t^r}^2}dt} = \Ocal\rbr{\eta^{-\frac{1}{r}}}.
\end{eqnarray}
Similarly, we can achieve
\begin{eqnarray}\label{eq:order_II}
\sum_{l=1}^\infty \frac{1}{\rbr{1 + \eta\zeta_l^{-1}}^2} &=& \Ocal\rbr{\eta^{-\frac{1}{r}}}, \\
\sum_{l=1}^\infty \frac{1}{{1 + \eta\zeta_l^{-1}}} &=& \Ocal\rbr{\eta^{-\frac{1}{r}}}.
\end{eqnarray}
Note also that $\sum_{l=1}^\infty \zeta_t^{-1}(f_l^*)^2=\|f^*\|_\Hcal^2<\infty$. Hence, the second term in \eqref{eq:expectation_expression} is also finite. 
Plugging~\eqref{eq:order_I} and~\eqref{eq:order_II} into~\eqref{eq:expectation_expression}, we achieve the conclusion. 
\end{proof}

Next, we characterize the approximation error due to parametrization in $L_2$ norm,
\begin{lemma}\label{lemma:kl_to_l2}
Under~\asmpref{asmp:q_param}, $\forall q\in \Pcal_w$ and $f\in \Hcal_k$, we have
$
\nbr{p_{f} - q}_2^2 \le 4\exp\rbr{2\kappa C_{\Hcal} + 2C_0}C_{\Pcal_w} KL\rbr{q||p_{f}}.
$
\end{lemma}
\begin{proof}
By~\lemref{lemma:bounded_pf}, we have $t = \frac{q}{p_{f}}\le 2\exp\rbr{\kappa C_{\Hcal} + C_0} C_{\Pcal_w}<\infty$. Denote $\Phi\rbr{t} = t\log t$, we have $\Phi''\rbr{t} = \frac{1}{t}\ge C_\Phi$ with $C_\Phi\defeq  \frac{1}{2\exp\rbr{\kappa C_{\Hcal} + C_0} C_{\Pcal_w}}$, and thus, $\Phi\rbr{t} - C_\Phi t^2$ is convex. Therefore, applying Jensen's inequality, we have 
\begin{eqnarray*}
&&\Phi\rbr{\EE_{p_{f}}\sbr{\frac{q}{p_{f}}}} - C_\Phi\rbr{\EE_{p_{f}}\sbr{\frac{q}{p_{f}}}}^2 \le \EE_{p_{f}}\sbr{\Phi\rbr{\frac{q}{p_{f}}} - C_\Phi \frac{q^2}{p^2_{f}}}\\
&\Rightarrow&C_\Phi\rbr{\int \frac{q^2\rbr{x}}{p_{f}\rbr{x}}dx - \rbr{\EE_{p_{f}}\sbr{\frac{q}{p_{f}}}}^2}\le  \EE_{p_{f}}\sbr{\Phi\rbr{\frac{q}{p_{f}}} } - \Phi\rbr{\EE_{p_{f}}\sbr{\frac{q}{p_{f}}}}\\
&\Rightarrow& C_\Phi\underbrace{\rbr{\int \frac{q^2\rbr{x}}{p_{f}\rbr{x}}dx - 1}}
_{\chi^2\rbr{q, p_{f}}} \le KL\rbr{q||p_{f}} - \Phi\rbr{1} = KL\rbr{q||p_{f}}.
\end{eqnarray*}

On the other hand, 
\begin{eqnarray*}
\chi^2\rbr{q, p_{f}} = \int \frac{\rbr{q\rbr{x} - p_{f}\rbr{x}}^2}{p_{f}\rbr{x}}dx\ge 2\exp\rbr{-\kappa C_{\Hcal} - C_0}\nbr{q - p_{f}}_2^2,
\end{eqnarray*}
which leads to the conclusion.

\end{proof}

We proceed the other part of the error, \ie, between $\ftil$ and $\htil$. 
\begin{lemma}\label{lemma:intermediate_II}
Under~\asmpref{asmp:kernel_property} and~\asmpref{asmp:exp_decay}, we have as $n\to\infty$ and $\eta\to 0$,
\begin{eqnarray*}
\nbr{\ftil - \htil}^2_{p^*} &=& o_{p^*}\rbr{n^{-1}\eta^{-\frac{1}{r}} + \eta} + C{\epsilon_{approx}},\\
\eta\nbr{\ftil - \htil}^2_{\Hcal} &=& o_{p^*}\rbr{\rbr{n^{-1}\eta^{-\frac{1}{r}} + \eta} + \epsilon_{approx}} + o_{p^*}{\rbr{n^{-1}\eta^{-\frac{1}{r}} + \eta}} + C\epsilon_{approx}.
\end{eqnarray*}
Therefore, 
\end{lemma}
\begin{eqnarray*}
\nbr{\ftil - f^*}^2_{p^*} &=& \Ocal_{p^*}\rbr{n^{-1}\eta^{-\frac{1}{r}} + \eta} + C{\epsilon_{approx}},\\
\eta\nbr{\ftil - f^*}^2_{\Hcal} &=& \Ocal_{p^*}{\rbr{n^{-1}\eta^{-\frac{1}{r}} + \eta}} + o_{p^*}\rbr{\rbr{n^{-1}\eta^{-\frac{1}{r}} + \eta} + \epsilon_{approx} } + C\epsilon_{approx}.
\end{eqnarray*}
\begin{proof}
Since $\rbr{\ftil, \tilde\nu, \qtil}$ are the optimal solutions to the primal-dual reformulation of the penalized MLE~\eqref{eq:double_dual_mle}, we have the first-order optimality condition: $\nabla_f \ell(\ftil,\tilde\nu, \qtil)=0$, which implies 
$\widehat\EE\sbr{k\rbr{x, \cdot}} - \EE_{\qtil}\sbr{k\rbr{x, \cdot} } - \eta\ftil = 0$. Hence,  
\begin{equation}\label{eq:opt_ftil}
\widehat\EE\sbr{\langle k\rbr{x, \cdot}, \ftil- \htil\rangle } - \EE_{\qtil}\sbr{\langle k\rbr{x, \cdot}, \ftil - \htil\rangle } - \eta\langle \ftil, \ftil-\htil \rangle_\Hcal = 0.
\end{equation}
Similarly, by the optimality of $\htil$ w.r.t. $\Ltil\rbr{h}$, we have
\begin{equation}\label{eq:opt_htil}
\widehat\EE\sbr{\langle k\rbr{x, \cdot}, \ftil-\htil\rangle }-\EE_{p^*}\sbr{\langle k\rbr{x, \cdot}, \ftil - \htil\rangle } - \langle\ftil - \htil, \htil - f^*\rangle_{p^*} - \eta\langle\htil, \ftil - \htil\rangle_\Hcal = 0.
\end{equation}
Combining the~\eqref{eq:opt_ftil} and~\eqref{eq:opt_htil}, we further obtain
\begin{eqnarray}\label{eq:opt_cond}
&&\EE_{\qtil}\sbr{\ftil\rbr{x} - \htil\rbr{x}} - \EE_{p_{\htil}}\sbr{\ftil\rbr{x} - \htil\rbr{x}} + \eta\nbr{\ftil - \htil}_{\Hcal}^2 \\
&&= \langle\ftil - \htil, \htil - f^*\rangle_{p^*} + \EE_{p^*}\sbr{\ftil\rbr{x} -\htil\rbr{x}} - \EE_{p_{\htil}}\sbr{\ftil\rbr{x} - \htil\rbr{x}} \nonumber\\
&\Rightarrow&\EE_{p_{\ftil}}\sbr{\ftil\rbr{x} - \htil\rbr{x}} - \EE_{p_{\htil}}\sbr{\ftil\rbr{x} - \htil\rbr{x}} + \eta\nbr{\ftil - \htil}_{\Hcal}^2 \nonumber\\
&&= \langle\ftil - \htil, \htil - f^*\rangle_{p^*} + \underbrace{\EE_{p^*}\sbr{\ftil\rbr{x} -\htil\rbr{x}} - \EE_{p_{\htil}}\sbr{\ftil\rbr{x} - \htil\rbr{x}}}_{\epsilon_1}\nonumber\\
&& + \underbrace{\EE_{p_{\ftil}}\sbr{\ftil\rbr{x} - \htil\rbr{x}}  - \EE_{\qtil}\sbr{\ftil\rbr{x} - \htil\rbr{x}}}_{\epsilon_2}.\nonumber
\end{eqnarray}
For $\epsilon_1$, denote $F\rbr{\theta} = \EE_{p_{f^* + \theta\rbr{\htil-f^*}/\varsigma}}\sbr{\ftil\rbr{x} - \htil\rbr{x}} - \EE_{p^*}\sbr{\ftil\rbr{x} -\htil\rbr{x}}$ with $\varsigma = \nbr{f^* - \htil}_{p^*} = o_{p^*}\rbr{1}$, then, apply Taylor expansion to $F\rbr{\theta}$ will lead to
\begin{equation}
F\rbr{\theta} = \frac{\theta}{\varsigma}
\rbr{1 + o_p(1)}\langle\ftil - \htil, \htil - f^*\rangle_{p^*}
\end{equation}
where $o_{p^*}\rbr{1}$ w.r.t. $\theta\rightarrow 0$. Therefore, 
\begin{equation}\label{eq:epsilon_1}
\EE_{p_{\htil}}\sbr{\ftil\rbr{x} - \htil\rbr{x}} - \EE_{p^*}\sbr{\ftil\rbr{x} -\htil\rbr{x}} = F\rbr{\varsigma} = \rbr{1 + o_p(1)}\langle\ftil - \htil, \htil - f^*\rangle_{p^*},
\end{equation}
as $\eta\rightarrow 0$ and $n\eta^{\frac{1}{r}}\to\infty$.

For $\epsilon_2$, by H\"{o}lder inequality,
\begin{eqnarray*}
\epsilon_2 &=& \EE_{p_{\ftil}}\sbr{\ftil\rbr{x} - \htil\rbr{x}}  - \EE_{\qtil}\sbr{\ftil\rbr{x} - \htil\rbr{x}} = \int_\Omega \frac{p_{\ftil}(x) - \qtil(x)}{p^*\rbr{x}}\rbr{\ftil\rbr{x} - \htil\rbr{x}}p^*\rbr{x}dx\\
&\le&\nbr{\ftil - \htil}_{p^*}\nbr{\frac{p_{\ftil}(x) - \qtil(x)}{p^*\rbr{x}}}_{p^*}.
\end{eqnarray*}
By~\lemref{lemma:bounded_pf}, we have 
% Due to the~\asmpref{asmp:prior}, 
$p^*\rbr{x} = \exp\rbr{f^* - A\rbr{f^*}} = \frac{\exp\rbr{f^*\rbr{x} - \log p_0\rbr{x}}}{\int_\Omega \exp\rbr{f^*\rbr{x} - \log p_0\rbr{x}} p_0\rbr{x}}p_0\rbr{x} $ with $f^*\in\Hcal_k$ and $\nbr{f^*}_{\Hcal}\le C_{f^*}$ 
% and $\nbr{\log p_0\rbr{x}}_\infty \le C_0$, 
implying $2\exp\rbr{-\kappa C_{f^*} - C_0}\le p^*\rbr{x}\le 2\exp\rbr{\kappa C_{f^*} + C_0}$. Therefore, we have
\begin{equation}\label{eq:epsilon_2}
\epsilon_2\le 2\exp\rbr{\kappa C_{f^*} + C_0}\nbr{p_{\ftil} - \qtil}_{2}\nbr{\ftil - \htil}_{p^*}.
\end{equation}

Given $\forall f\in\Fcal$ fixed, we have 
\begin{eqnarray}
\min_{q\in\Pcal}\max_{\nu\in \Fcal} - \EE_q\sbr{f}  +\frac{1}{\lambda }\rbr{\EE_q\sbr{\nu} - \EE_{p_0}\sbr{\exp\rbr{\nu}}} + 1 = \min_{q\in\Pcal} - \EE_q\sbr{f}  + \frac{1}{\lambda} KL\rbr{q||p_0} = \min_{q\in\Pcal}\frac{1}{\lambda}KL\rbr{q||p_f}.
\end{eqnarray}
In ideal case where the $\Pcal$ and the domain of $\nu$ is flexible enough, we have $KL\rbr{q||p_f} = 0$. However, due to parametrization, we introduce extra approximate error, denoting as $\epsilon_{approx}\defeq \sup_{f\in\Hcal_k}\inf_{q\in\Pcal_w}KL\rbr{q||p_{f}}$. As we can see, as the parametrized family of $q$ and $\nu$ becomes more and more flexible, the $\epsilon_{approx}\rightarrow 0$. 

Therefore, by~\lemref{lemma:kl_to_l2}, we obtain 
$$
\epsilon_2\le 8{C_{\Pcal_w}}{\exp\rbr{3\kappa C_{\Hcal} + 3C_0}}\epsilon^{\frac{1}{2}}_{approx}\nbr{\ftil - \htil}_{p^*}.
$$

On the other hand, we define $D\rbr{\theta} = \EE_{p_{\htil + \theta\rbr{\ftil - \htil}}}\sbr{\ftil - \htil}$, notice that $D'\rbr{\theta} = \nbr{\ftil - \htil}^2_{p_{\htil + \theta\rbr{\ftil - \htil}}}$, by the mean value theorem, we can obtain that
\begin{eqnarray}
\EE_{p_{\ftil}}\sbr{\ftil\rbr{x} - \htil\rbr{x}} - \EE_{p_{\htil}}\sbr{\ftil\rbr{x} - \htil\rbr{x}}  = D\rbr{1} - D\rbr{0} = D'\rbr{\theta} = \nbr{\ftil - \htil}^2_{p_{\htil + \theta\rbr{\ftil - \htil}}}
\end{eqnarray}
with $\theta\in[0, 1]$.  \citet{GuQiu93} shows that when $\forall f\in \Hcal_k$ is uniformly bounded, then, $c\nbr{\cdot}_{p^*}\le \nbr{\cdot}_{p_{\htil + \theta\rbr{\ftil - \htil}}}$, $\theta\in[0, 1]$, which is the true under the~\asmpref{asmp:kernel_property}. 

Plugging \eqref{eq:epsilon_1} and~\eqref{eq:epsilon_2} into~\eqref{eq:opt_cond}, we achieve  
$$
c\nbr{\ftil - \htil}^2_{p^*} + \eta\nbr{\ftil - \htil}^2_{\Hcal}\le o_p\rbr{\langle\ftil - \htil, \htil - f^*\rangle_{p^*}} +  8{C_{\Pcal_w}}{\exp\rbr{3\kappa C_{\Hcal} + 3C_0}}\epsilon^{\frac{1}{2}}_{approx}\nbr{\ftil - \htil}_{p^*},
$$
which leads to the first part in the conclusion. Combining with the~\lemref{lemma:intermediate_I}, we obtain the second part of the conclusion.

\end{proof}

Now, we are ready for proving the main theorem about the statistical consistency.

\noindent{\bf Theorem~\ref{thm:consistency}}
\emph{
Assume the spectrum of kernel $k\rbr{\cdot, \cdot}$ decays sufficiently homogeneously in rate $l^{-r}$. With some other mild assumptions listed in~\appref{appendix:subsec:consistency}, we have as $\eta\rightarrow 0$ and $n\eta^{\frac{1}{r}}\to\infty$,
\begin{eqnarray*}
KL\rbr{p^*||p_{\ftil}} + KL\rbr{p_{\ftil}||p^*} = \Ocal_{p^*}\rbr{n^{-1}\eta^{-\frac{1}{r}} + \eta + \epsilon_{approx}},
\end{eqnarray*}
where $\epsilon_{approx}\defeq\sup_{f\in\Fcal}\inf_{q\in\Pcal_w}KL\rbr{q||p_{f}}$ denotes the approximate error  due to the parametrization of $\qtil$ and $\tilde\nu$.
% $\epsilon_{approx}\defeq \sup_{f\in \Fcal}\inf_{q\in\Pcal_w} \nbr{p_f -q}_{p^*}$. 
Therefore, when setting $\eta = \Ocal\rbr{n^{-\frac{r}{1+r}}}$, $p_{\ftil}$ converge
s to $p^*$  in terms of Jensen-Shannon divergence at rate 
$
\Ocal_{p^*}\rbr{n^{-\frac{r}{1+r}} + \epsilon_{approx}}.
$
}
\begin{proof}
Recall the $\rbr{\ftil, \qtil}$ is the optimal solution to~\eqref{eq:primal_CD}, we have the first-order optimality condition as 
\begin{equation}\label{eq:opt_ftil_f*}
\widehat\EE\sbr{\langle k\rbr{x, \cdot}, \ftil- f^*\rangle } - \EE_{\qtil}\sbr{\langle k\rbr{x, \cdot}, \ftil - f^*\rangle } - \eta\langle \ftil, \ftil-f^* \rangle_\Hcal = 0,
\end{equation}
which leads to
\begin{eqnarray*}
\EE_{p_{\ftil}}\sbr{\langle k\rbr{x, \cdot}, \ftil - f^*\rangle } &=& \widehat\EE\sbr{\langle k\rbr{x, \cdot}, \ftil- f^*\rangle } - \eta\langle \ftil, \ftil-f^* \rangle_\Hcal\\
&&+ \underbrace{\EE_{p_{\ftil}}\sbr{\langle k\rbr{x, \cdot}, \ftil - f^*\rangle } - \EE_{\qtil}\sbr{\langle k\rbr{x, \cdot}, \ftil - f^*\rangle }}_{\epsilon_3}. 
\end{eqnarray*}
Then, we can rewrite the Jensen-Shannon divergence
\begin{eqnarray*}
KL\rbr{p^*||p_{\ftil}} + KL\rbr{p_{\ftil}||p^*} &=&\EE_{p_{\ftil}}\sbr{\langle k\rbr{x, \cdot}, \ftil - f^*\rangle }- \EE_{p^*}\sbr{\langle k\rbr{x, \cdot}, \ftil - f^*\rangle }\\
&=& \epsilon_3  + \eta\langle \ftil, f^* - \ftil \rangle_\Hcal + \widehat\EE\sbr{\langle k\rbr{x, \cdot}, \ftil- f^*\rangle } - \EE_{p^*}\sbr{\langle k\rbr{x, \cdot}, \ftil - f^*\rangle }
\end{eqnarray*}

Similar to the bound of $\epsilon_2$, we have
$$
\epsilon_3 \le 8{C_{\Pcal_w}}{\exp\rbr{3\kappa C_{\Hcal} + 3C_0}}\epsilon^{\frac{1}{2}}_{approx}\nbr{\ftil - f^*}_{p^*} = \Ocal\rbr{\epsilon^{\frac{1}{2}}_{approx}\sqrt{n^{-1}\eta^{-\frac{1}{r}} + \eta}} = \Ocal\rbr{\epsilon_{approx} + \rbr{n^{-1}\eta^{-\frac{1}{r}} + \eta}}.
$$
Moreover, with Cauchy-Schwarz inequality,
$$
\eta\langle \ftil, f^* - \ftil\rangle_\Hcal\le \eta \nbr{\ftil}_{\Hcal}\nbr{\ftil-f^*}_\Hcal,
$$
$$
\eta\nbr{\ftil}_\Hcal\le 2\eta\nbr{f^*}_\Hcal + 2\eta\nbr{\ftil - f^*}_\Hcal
$$
Applying the conclusion in~\lemref{lemma:intermediate_II} and the fact that $\|f^*\|_\Hcal\leq C_{f^*}$, 
we obtain that
$$
\eta\langle \ftil, f^*-\ftil \rangle_\Hcal = \Ocal(\eta).
$$
Finally, for the term
$$
\widehat\EE\sbr{\langle k\rbr{x, \cdot}, \ftil- f^*\rangle } - \EE_{p^*}\sbr{\langle k\rbr{x, \cdot}, \ftil - f^*\rangle },
$$
we rewrite $\ftil$ and $f^*$ in the form of $\psi$ as $\sum_{l=1}^\infty \rbr{\ftil_l - f^*_l}\alpha_l$. Then, apply Cauchy-Schwarz inequality, 
$$
\sum_{l=1}^\infty \abr{\rbr{\ftil_l - f^*_l}\alpha_l} \le\rbr{\sum_{l=1}^\infty a_l^2\rbr{\ftil_l - f^*_l}^2}^{\frac{1}{2}}\rbr{\sum_{l=1}^\infty \rbr{\frac{\alpha_l}{a_l}}^2}^{\frac{1}{2}}
$$
where $a_l^2 = 1 + \eta\zeta_l^{-1}$. Then, 
$$
\sum_{l=1}^\infty a_l^2\rbr{\ftil_l - f^*_l}^2 = \nbr{\ftil - f^*}^2_{p^*} + 
\eta\nbr{\ftil - f^*}^2_{\Hcal} = \Ocal_{p^*}{\rbr{n^{-1}\eta^{-\frac{1}{r}} + \eta + \epsilon_{approx}}} + o_{p^*}\rbr{\rbr{n^{-1}\eta^{-\frac{1}{r}} + \eta} + \epsilon_{approx}}.
$$
On the other hand, by~\lemref{lemma:intermediate_I}
$$
\EE\sbr{\sum_{l=1}^\infty \rbr{\frac{\alpha_l}{a_l}}^2} = \Ocal\rbr{n^{-1}\eta^{-\frac{1}{r}}}.
$$

Combining these bounds, we achieve the conclusion that
$$
KL\rbr{p^*||p_{\ftil}} + KL\rbr{p_{\ftil}||p^*} = \Ocal_{p^*}{\rbr{n^{-1}\eta^{-\frac{1}{r}} + \eta + \epsilon_{approx}}}.
$$
The second conclusion is straightforward by balancing $\eta$.

\end{proof}

%%%%%%%%%%%%%%%%%%%%%%%%%%%%%%%%%%%%%%%%%%%%%%%%%%%%%%%%%%%%%%%%%%%%%%%%%%%%%%%%%%%%%%%%%%%%
\section{MLE with Random Feature Approximation}\label{appendix:random_feature}
%%%%%%%%%%%%%%%%%%%%%%%%%%%%%%%%%%%%%%%%%%%%%%%%%%%%%%%%%%%%%%%%%%%%%%%%%%%%%%%%%%%%%%%%%%%%

The memory cost is the main bottleneck for applying the kernel methods to large-scale problems. The random feature~\citep{RahRec08,DaiXieHeLiaEtAl14,Bach15} can be utilized for scaling up kernel methods. In this section, we will propose the variant of the proposed algorithm with random feature approximation.

For arbitrary positive definite kernel, $k(x, x)$, there exists a measure $\PP$ on $\Xcal$, such that $k(x, x') = \int \phi_w(x)\phi_w(x')d\PP(w)$~\cite{Devinatz53,HeiBou04}, where $\phi_w(x):\Xcal\rightarrow \RR$ from $L_2(\Xcal, \PP)$. Therefore, we can approximate the function $f\in \Hcal$ with Monte-Carlo approximation $\hat f \in \widehat\Hcal^r=\{\sum_{i=1}^r\beta_i\phi_{\omega_i}(\cdot)|\|\beta\|_2\leq C\}$ where $\{w_i\}_{i=1}^r$ sampled from $\PP(\omega)$. The $\cbr{\phi_{\omega_i}(\cdot)}_{i=1}^r$ are called random features~\cite{RahRec09}. With such approximation, we will apply the stochastic gradient to learn $\cbr{\beta_i}_{i=1}^r$. For simplicity, we still consider the saddle-point reformulation of MLE for exponential families. However, the algorithm applies to general flows and conditional models too. 

Plug the approximation of $\fhat\rbr{\cdot} = \sum_{i=1}^r\beta_i\phi_{\omega_i}(\cdot) = \beta_f^\top \Phi\rbr{\cdot}$ and $\widehat\nu = \beta_\nu^\top \Phi\rbr{\cdot}$ into the optimization~\eqref{eq:double_dual_mle} and denote $\beta = \cbr{\beta_f, \beta_\nu}$, we have
\begin{equation}\label{eq:double_dual_random_mle}
\min_{w_g}\,\Lbar\rbr{w_g} := \max_{\beta_f,\beta_\nu}\,\,\underbrace{\tilde\ell\rbr{\beta_f, \beta_\nu, w_g} - \frac{\eta}{2}\nbr{\beta_f}^2}_{\bar\ell\rbr{\beta_f, \beta_\nu, w_g}}.
\end{equation}
Therefore, we have the random feature variant of~\algref{alg:sgd_f}
\begin{algorithm}[h]
\caption{Stochastic Gradients for $\beta_f^*$ and $\beta_\nu^*$}
  \begin{algorithmic}[1]\label{alg:sgd_beta}
    \FOR{$k=1,\ldots, K$}
      \STATE Sample ${\xi\sim p\rbr{\xi}}$, and generate $x = g\rbr{\xi}$.
      \STATE Sample ${x'\sim p_0\rbr{x}}$.
      \STATE Compute stochastic function gradient w.r.t. $\beta_f$ and $\beta_\nu$.
      \STATE Decay the stepsize $\tau_k$.
      \STATE Update $\beta_f^k$ and $\beta_\nu^k$ with the stochastic gradients. 
    \ENDFOR
    \STATE Output $\beta_f^K, \beta_\nu^K$.
  \end{algorithmic}
\end{algorithm}

With the obtained $\rbr{\beta_f^K, \beta_\nu^K}$, the~\algref{alg:sgd_double_dual} will keep almost the same, except in Step 2 call~\algref{alg:sgd_beta} instead. 

One can also adapt the random feature $\cbr{\omega_i}_{i=1}^r$ by stochastic gradient back-propagation~(BP) too in~\algref{alg:sgd_beta}. Then, the $\nu$ is equivalent to parametrized by a two-layer MLP neural networks. Similarly, we can deepen the neural networks for $\nu$, and the parameters can still be trained by BP. 

%%%%%%%%%%%%%%%%%%%%%%%%%%%%%%%%%%%%%%%%%%%%%%%%%%%%%%%%%%%%%%%%%%%%%%%%%%%%%%%%%%%%%%%%%%%%
\section{Computational Cost Analysis}\label{appendix:comp_cost}
%%%%%%%%%%%%%%%%%%%%%%%%%%%%%%%%%%%%%%%%%%%%%%%%%%%%%%%%%%%%%%%%%%%%%%%%%%%%%%%%%%%%%%%%%%%%

Following the notations in the paper, the computational cost for~\algref{alg:sgd_f} will be $\Ocal\rbr{K^2d}$. Then, the total cost for~\algref{alg:sgd_double_dual} will be $\Ocal\rbr{L\rbr{K^2d + BKd}}$ with $B$ as batchsize and $L$ as the number of iterations. If we stop the algorithm after scanning the dataset, i.e., $BL=N$, we have the cost as $\Ocal\rbr{NK^2d}$, which is more efficient comparing to score matching based estimator.

%%%%%%%%%%%%%%%%%%%%%%%%%%%%%%%%%%%%%%%%%%%%%%%%%%%%%%%%%%%%%%%%%%%%%%%%%%%%%%%%%%%%%%%%%%%%
\section{More Experimental Results}\label{appendix:more_exp}
%%%%%%%%%%%%%%%%%%%%%%%%%%%%%%%%%%%%%%%%%%%%%%%%%%%%%%%%%%%%%%%%%%%%%%%%%%%%%%%%%%%%%%%%%%%%

We provide more empirical experimental results here. We further illustrate the convergence of the algorithm on the $2$-dimensional \texttt{grid} and \texttt{two moons} in~\figref{fig:alg_conv_more}. 

\begin{figure*}[htb]
\centering
  \begin{tabular}{cccc}
    \includegraphics[width=0.24\textwidth]{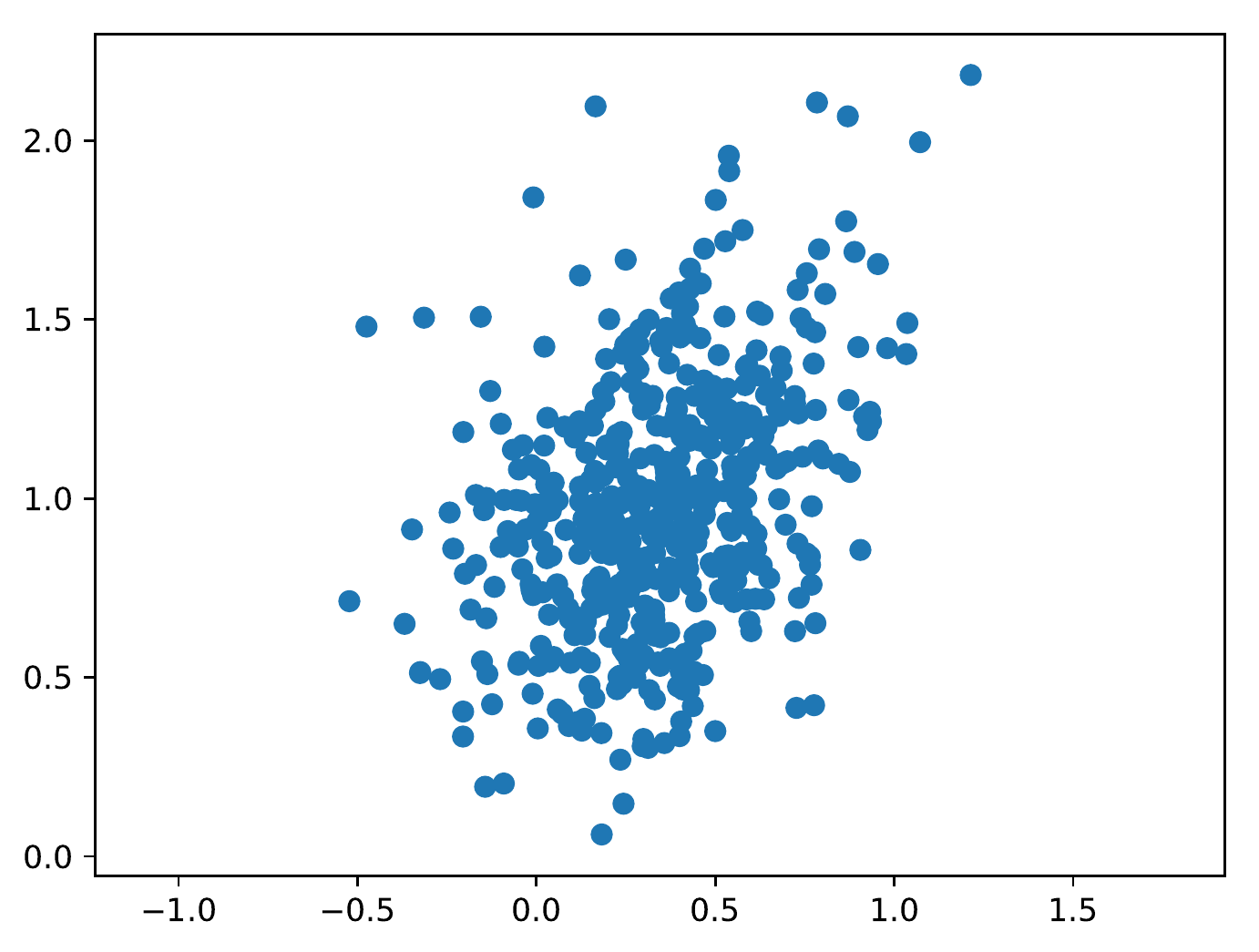}&\hspace{-5mm}
    \includegraphics[width=0.24\textwidth]{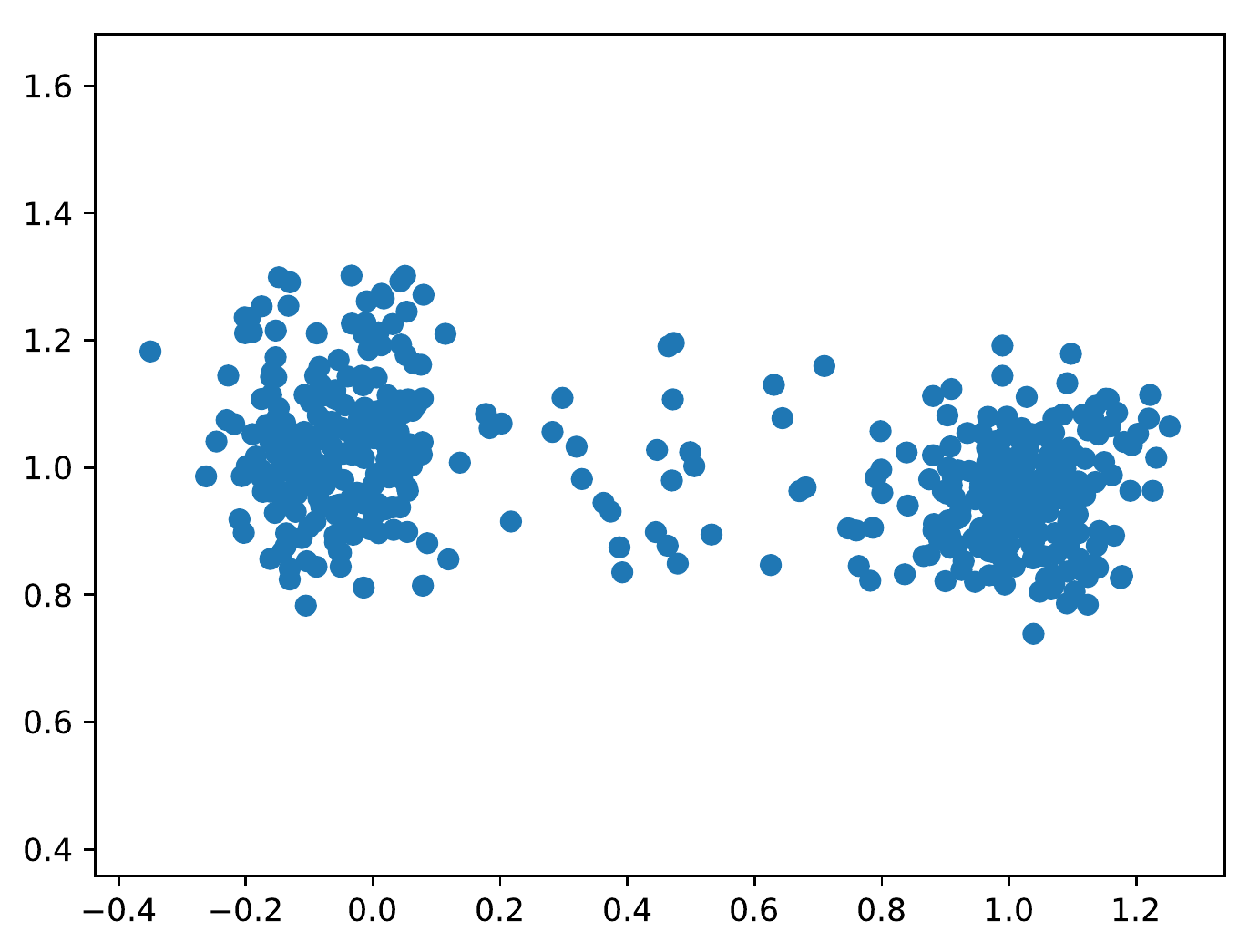}&\hspace{-5mm}
    \includegraphics[width=0.24\textwidth]{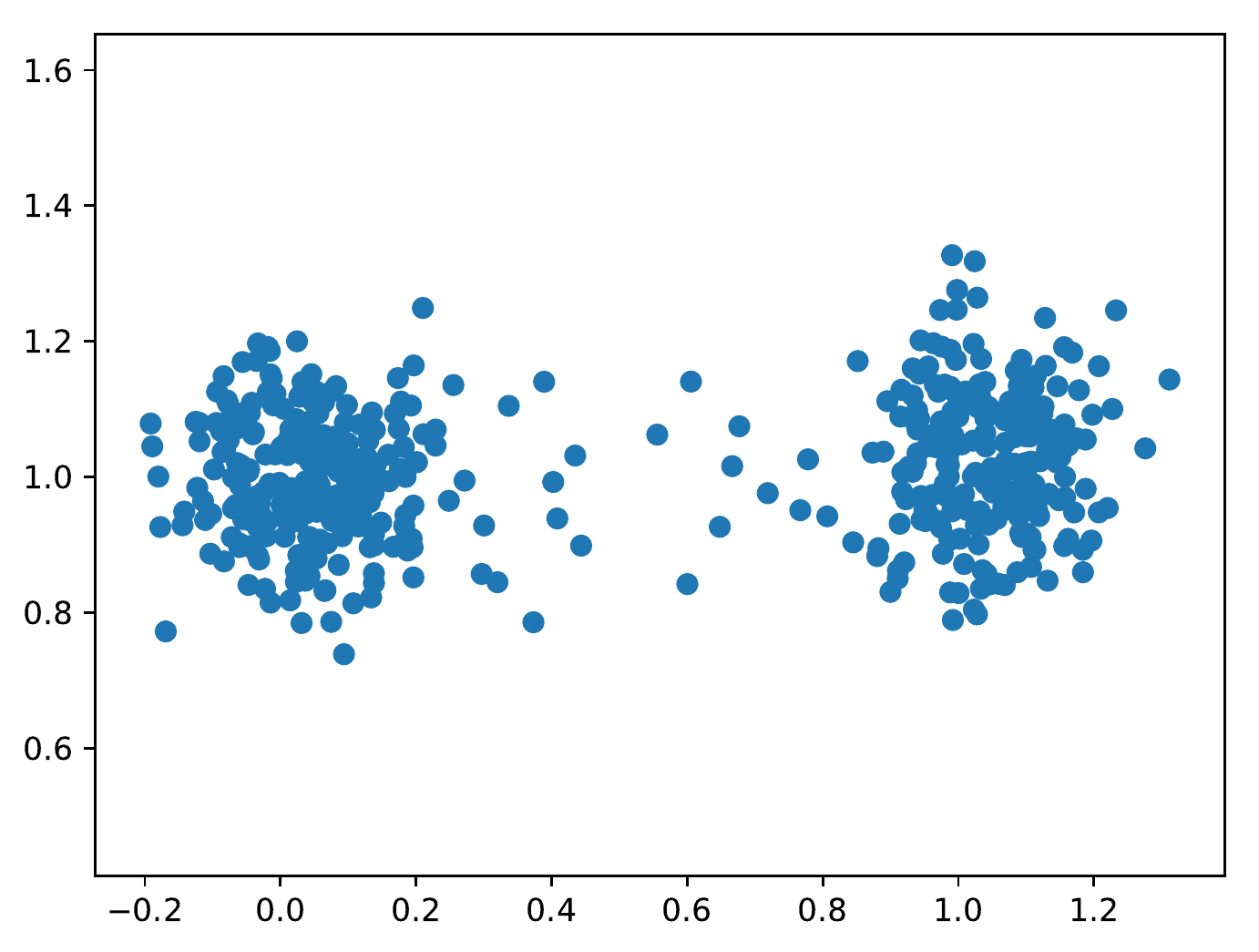}&\hspace{-5mm}
    \includegraphics[width=0.24\textwidth]{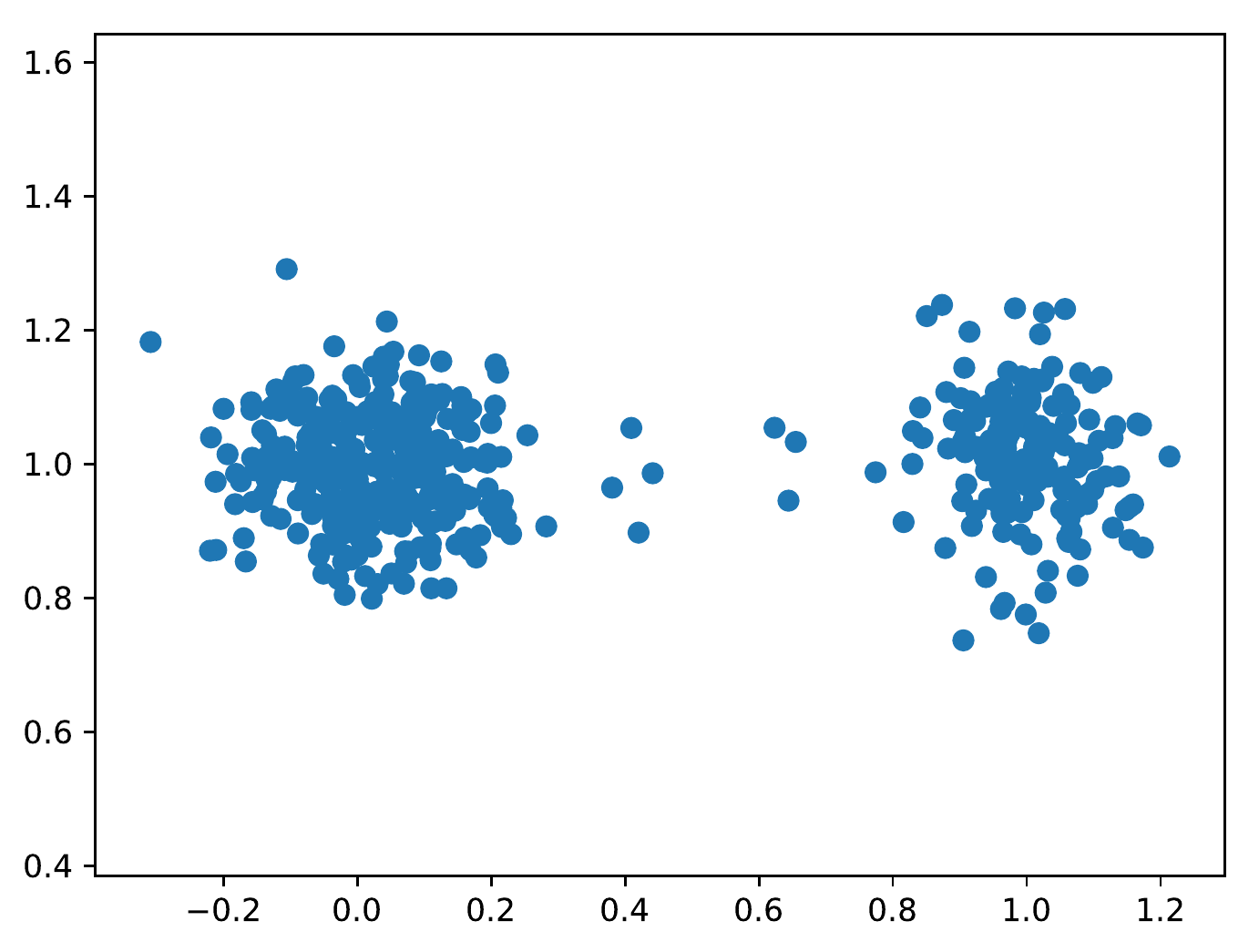}\\
    (a) initialization &(b) $500$-th iteration  &(c) $1000$-th iteration &(d) $2000$-th iteration\\
  \end{tabular}
  \begin{tabular}{cccc}
    \includegraphics[width=0.24\textwidth]{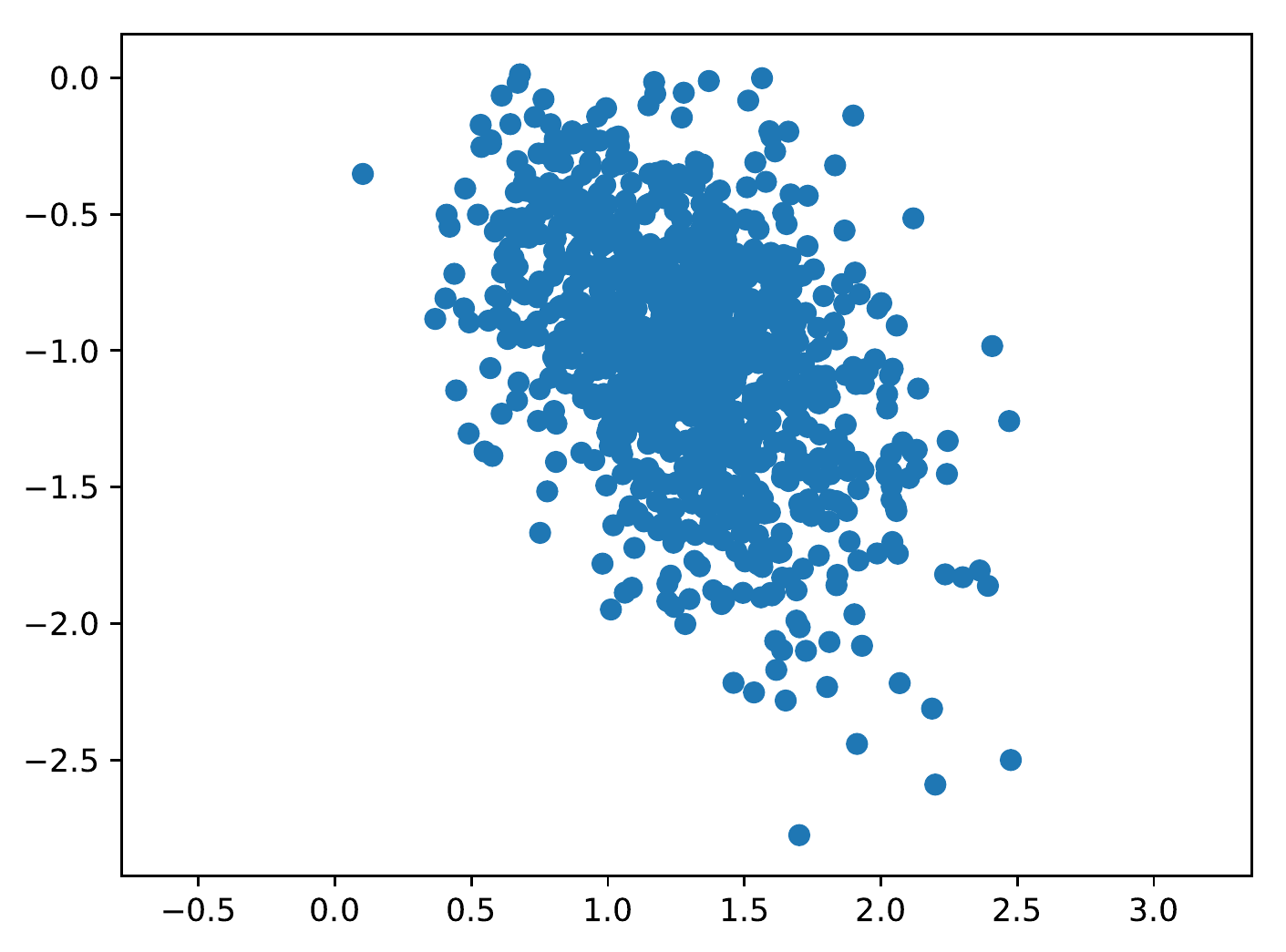}&\hspace{-5mm}
    \includegraphics[width=0.24\textwidth]{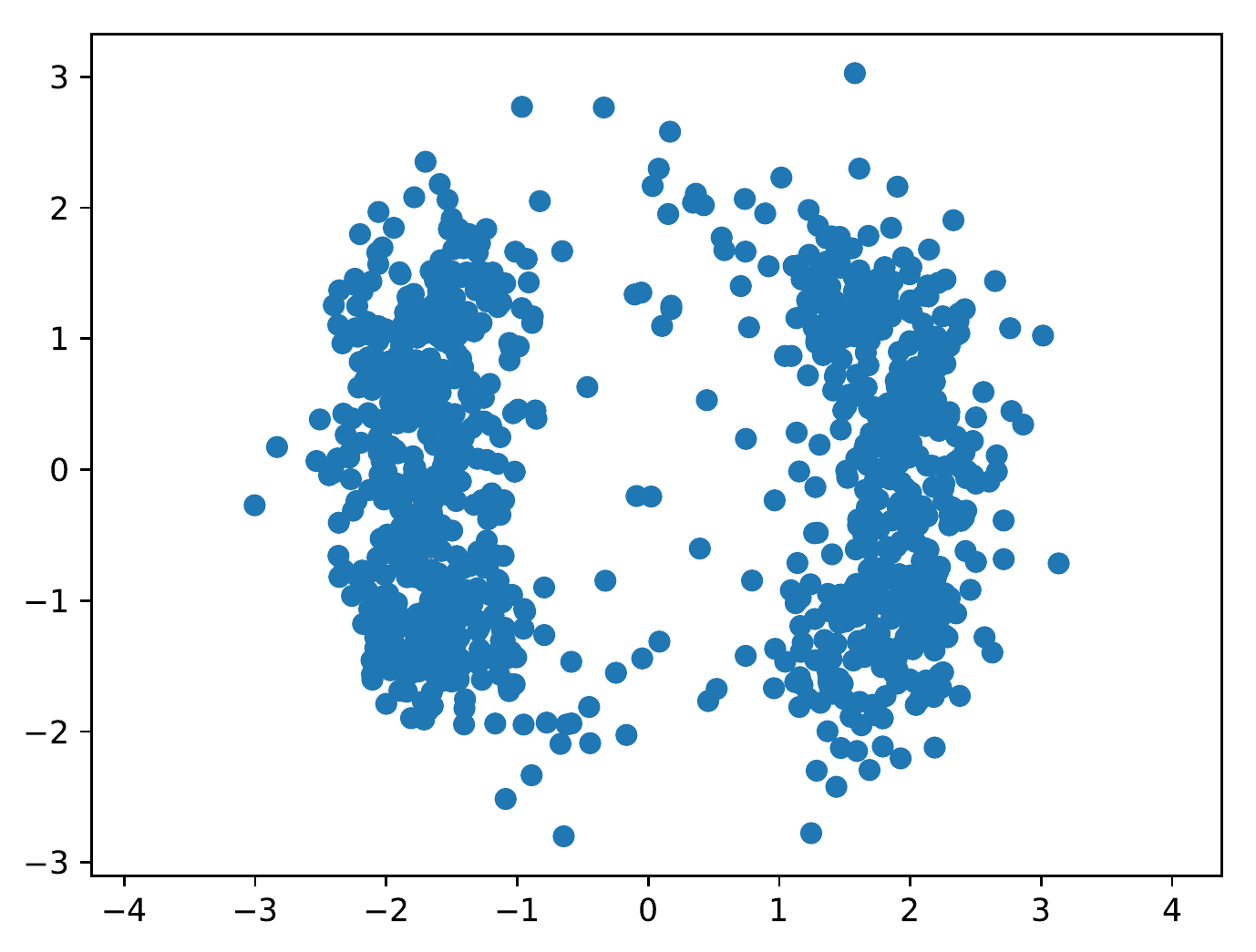}&\hspace{-5mm}
    \includegraphics[width=0.24\textwidth]{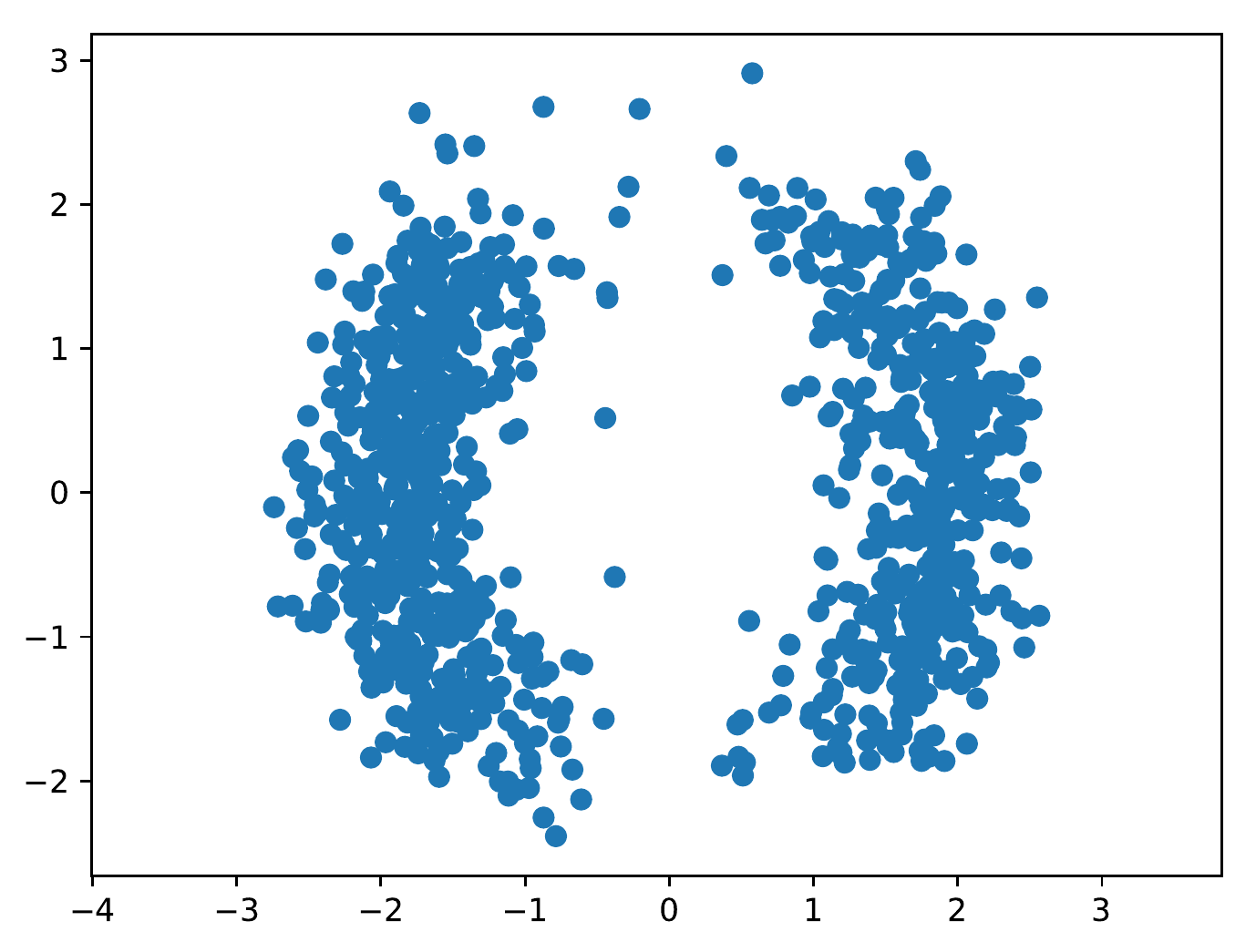}&\hspace{-5mm}
    \includegraphics[width=0.24\textwidth]{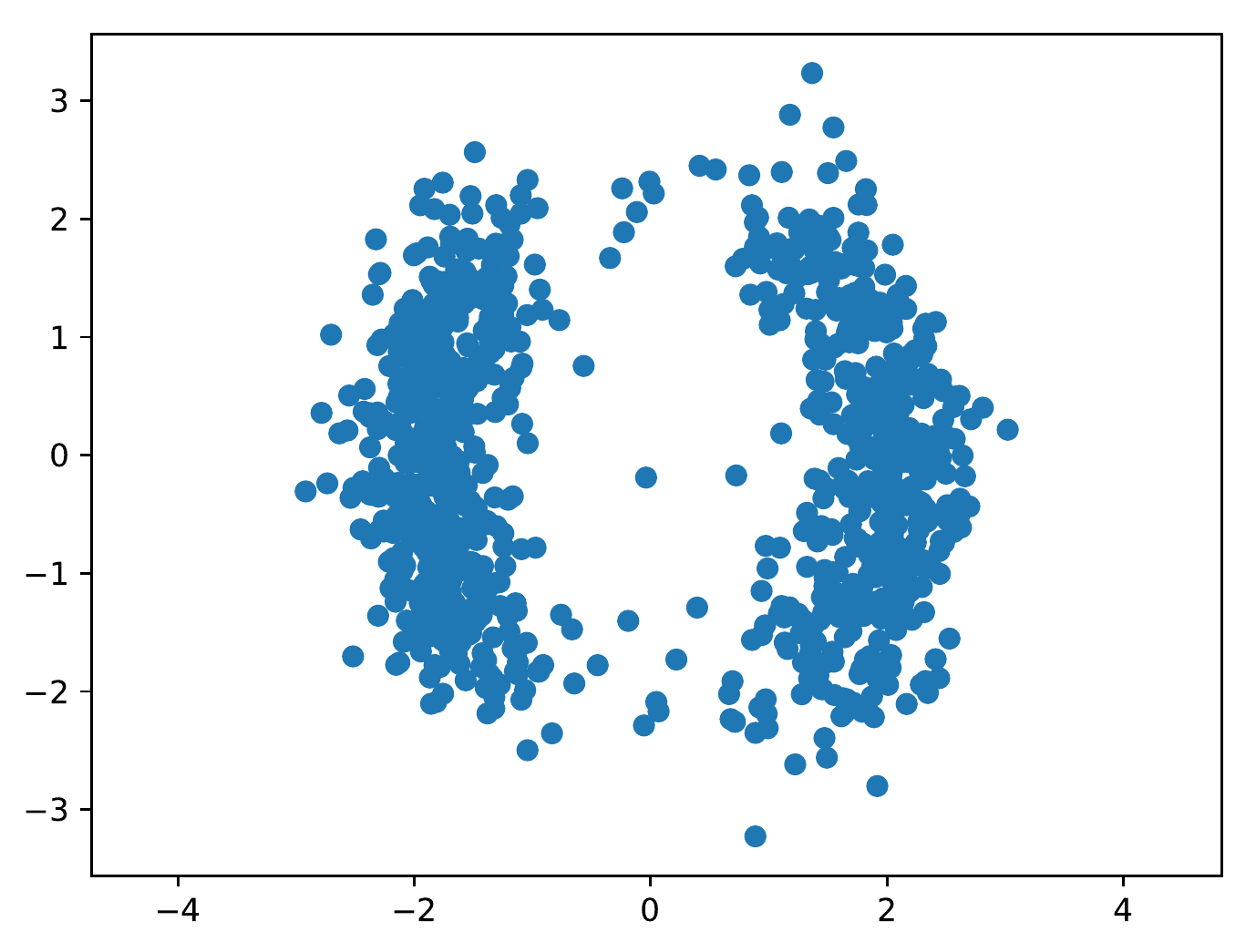}\\
    (a) initialization &(b) $500$-th iteration  &(c) $1000$-th iteration &(d) $2000$-th iteration\\
  \end{tabular}
  \caption{The DDE estimators on $2$-dimensional \texttt{grid} and \texttt{two moons} datasets in each iteration. The {\color{blue}blue} points are sampled from the learned dual distribution. The algorithm starts with random initialization. With the algorithm proceeds, the learned distribution converges to the ground-truth target distributions. }
  \label{fig:alg_conv_more}
\end{figure*}

\section{Implementation Details}
\label{appendix:impl_details}

In this section, we will provide more details about algorithm implementation. Our implementation is based on PyTorch, and is open sourced at \url{https://github.com/Hanjun-Dai/dde}. 

To optimize with the double min-max form, we adopt the following training schema. For every gradient update of the exponential family model $f$, 5 updates of sampler $g_{w_g}$ will be performed. And for each update of $g_{w_g}$, 3 updates of $\nu$ will be performed. Generally, the inner terms of the objective function will get more updates. 

For unconditional experiments on synthetic datasets, we use dimension 128 for both the hidden layers of MLP networks, as well as $\xi$. The number of layers for generator $g_{w_g}$ and $\nu$ are tuned in the range of $\{3, 4, 5\}$. For conditional experiments on real-world datasets, we use 3 layers for both $g_{w_g}$ and $\nu$, since the dataset is relatively small. To make the training stable, we also clip the gradients of all updates by the norm of 5. 

The hyperparameters, \eg, stepsize, kernel parameters, and weights of the penalty, are tuned by cross-validation.

\end{appendix}


\begin{thebibliography}{43}
\providecommand{\natexlab}[1]{#1}
\providecommand{\url}[1]{\texttt{#1}}
\expandafter\ifx\csname urlstyle\endcsname\relax
  \providecommand{\doi}[1]{doi: #1}\else
  \providecommand{\doi}{doi: \begingroup \urlstyle{rm}\Url}\fi

\bibitem[Altun and Smola(2006)]{AltSmo06}
Yasemin Altun and Alex Smola.
\newblock Unifying divergence minimization and statistical inference via convex
  duality.
\newblock In \emph{COLT}, pages 139--153, 2006.

\bibitem[Arbel and Gretton(2017)]{ArbGre17}
Michael Arbel and Arthur Gretton.
\newblock Kernel conditional exponential family.
\newblock In \emph{AISTATS}, pages 1337--1346, 2017.

\bibitem[Arjovsky et~al.(2017)Arjovsky, Chintala, and Bottou]{ArjChiBot17}
Martin Arjovsky, Soumith Chintala, and L{\'e}on Bottou.
\newblock Wasserstein GAN.
\newblock In \emph{ICML}, 2017.

\bibitem[Bach(2015)]{Bach15}
Francis~R. Bach.
\newblock On the equivalence between quadrature rules and random features.
\newblock \emph{Journal of Machine Learning Research}, 18:\penalty0 714--751,
  2017.

\bibitem[Barron and Sheu(1991)]{BarShe91}
Andrew~R Barron and Chyong-Hwa Sheu.
\newblock Approximation of density functions by sequences of exponential
  families.
\newblock \emph{The Annals of Statistics}, pages 1347--1369, 1991.

\bibitem[Bi{\'n}kowski et~al.(2018)Bi{\'n}kowski, Sutherland, Arbel, and
  Gretton]{BinSutArbArt18}
Miko{\l}aj Bi{\'n}kowski, Dougal~J Sutherland, Michael Arbel, and Arthur
  Gretton.
\newblock Demystifying MMD GANs.
\newblock In \emph{ICLR}, 2018.

\bibitem[Brown(1986)]{Brown86}
Lawrence~D. Brown.
\newblock \emph{Fundamentals of Statistical Exponential Families}, volume~9 of
  \emph{Lecture notes-monograph series}.
\newblock Institute of Mathematical Statistics, Hayward, Calif, 1986.

\bibitem[Canu and Smola(2006)]{CanSmo06}
Stephane Canu and Alex Smola.
\newblock Kernel methods and the exponential family.
\newblock \emph{Neurocomputing}, 69\penalty0 (7-9):\penalty0 714--720, 2006.

\bibitem[Dai et~al.(2014)Dai, Xie, He, Liang, Raj, Balcan, and
  Song]{DaiXieHeLiaEtAl14}
Bo Dai, Bo Xie, Niao He, Yingyu Liang, Anant Raj,  Maria-Florina Balcan, and Le Song. 
\newblock Scalable kernel methods via doubly stochastic gradients.
\newblock In \emph{NeurIPS}, 2014.

\bibitem[Dai et~al.(2016)Dai, He, Pan, Boots, and Song]{DaiHePanBooetal16}
Bo~Dai, Niao He, Yunpeng Pan, Byron Boots, and Le~Song.
\newblock Learning from conditional distributions via dual embeddings.
\newblock In \emph{AISTATS}, pages 1458--1467, 2017.


\bibitem[Dai et~al.(2017)Dai, Shaw, Li, Xiao, He, Liu, Chen, and
  Song]{DaiShaLiXiaHeetal17}
Bo~Dai, Albert Shaw, Lihong Li, Lin Xiao, Niao He, Zhen Liu, Jianshu Chen, and
  Le~Song.
\newblock SBEED: Convergent reinforcement learning with nonlinear function
  approximation.
\newblock In \emph{ICML}, pages 1133--1142, 2018.


\bibitem[Devinatz(1953)]{Devinatz53}
A.~Devinatz.
\newblock Integral representation of pd functions.
\newblock \emph{Trans. AMS}, 74\penalty0 (1):\penalty0 56--77, 1953.

\bibitem[Dinh et~al.(2016)Dinh, Sohl-Dickstein, and Bengio]{DinSohBen16}
Laurent Dinh, Jascha Sohl-Dickstein, and Samy Bengio.
\newblock Density estimation using real NVP.
\newblock In \emph{ICLR}, 2017.

\bibitem[Dud\'{\i}k et~al.(2007)Dud\'{\i}k, Phillips, and
  Schapire]{DudPhiSch07}
Miroslav Dud\'{\i}k, Steven~J. Phillips, and Robert~E. Schapire.
\newblock Maximum entropy density estimation with generalized regularization
  and an application to species distribution modeling.
\newblock \emph{Journal of Machine Learning Research}, 8:\penalty0 1217--1260,
  2007.


\bibitem[Ekeland and Temam(1999)]{EkeTem99}
Ivar Ekeland and Roger Temam.
\newblock \emph{Convex analysis and variational problems}, volume~28.
\newblock SIAM 1999.

\bibitem[Fukumizu(2009)]{Fukumizu09}
Kenji ~Fukumizu.
\newblock \emph{Exponential manifold by reproducing kernel Hilbert spaces},
  page 291–306.
\newblock Cambridge University Press, 2009.

\bibitem[Ghadimi and Lan(2013)]{GhaLan13}
Saeed Ghadimi and Guanghui Lan.
\newblock Stochastic first-and zeroth-order methods for nonconvex stochastic
  programming.
\newblock \emph{SIAM Journal on Optimization}, 23\penalty0 (4):\penalty0
  2341--2368, 2013.

\bibitem[Goodfellow et~al.(2014)Goodfellow, Pouget-Abadie, Mirza, Xu,
  Warde-Farley, Ozair, Courville, and Bengio]{GooPouMirXuetal14}
Ian Goodfellow, Jean Pouget-Abadie, Mehdi Mirza, Bing Xu, David Warde-Farley,
  Sherjil Ozair, Aaron Courville, and Yoshua Bengio.
\newblock Generative adversarial nets.
\newblock In \emph{NeurIPS}, pages
  2672--2680, 2014.

\bibitem[Gu and Qiu(1993)]{GuQiu93}
Chong Gu and Chunfu Qiu.
\newblock Smoothing spline density estimation: Theory.
\newblock \emph{The Annals of Statistics}, pages 217--234, 1993.

\bibitem[Hein and Bousquet(2004)]{HeiBou04}
Matthias Hein and Olivier Bousquet.
\newblock Kernels, associated structures, and generalizations.
\newblock Technical Report 127, Max Planck Institute for Biological
  Cybernetics, 2004.

\bibitem[Hinton(2002)]{Hinton02}
Geoffrey~E. Hinton.
\newblock Training products of experts by minimizing contrastive divergence.
\newblock \emph{Neural Computation}, 14\penalty0 (8):\penalty0 1771--1800,
  2002.

\bibitem[Hiriart-Urruty and Lemar{\'e}chal(2012)]{HirLem12}
Jean-Baptiste Hiriart-Urruty and Claude Lemar{\'e}chal.
\newblock \emph{Fundamentals of convex analysis}.
\newblock Springer Science \& Business Media, 2012.

\bibitem[Hyv\"arinen(2005)]{Hyvarinen05}
Aapo Hyv\"arinen.
\newblock Estimation of non-normalized statistical models using score matching.
\newblock \emph{Journal of Machine Learning Research}, 6:\penalty0 695--709,
  2005.

\bibitem[Kingma and Welling(2013)]{KinWel13}
Diederik~P Kingma and Max Welling.
\newblock Auto-encoding variational Bayes.
\newblock In \emph{ICLR}, 2014.

\bibitem[Kingma et~al.(2016)Kingma, Salimans, Jozefowicz, Chen, Sutskever, and
  Welling]{KinSalJozCheetal16}
Diederik~P Kingma, Tim Salimans, Rafal Jozefowicz, Xi~Chen, Ilya Sutskever, and
  Max Welling.
\newblock Improved variational inference with inverse autoregressive flow.
\newblock In \emph{NeurIPS}, pages
  4743--4751, 2016.

\bibitem[Kivinen et~al.(2004)Kivinen, Smola, and Williamson]{KivSmoWil04}
Jyrki Kivinen, Alex Smola, and Robert C. Williamson.
\newblock Online learning with kernels.
\newblock \emph{{IEEE} Transactions on Signal Processing}, 52\penalty0 (8), Aug
  2004.

\bibitem[K{\"o}nig(1986)]{Konig86}
Hermann ~K{\"o}nig.
\newblock \emph{Eigenvalue Distribution of Compact Operators}.
\newblock Birkh{\"a}user, Basel, 1986.

\bibitem[Li et~al.(2017)Li, Chang, Cheng, Yang, and Poczos]{LiChaYuYanetal17}
Chun-Liang Li, Wei-Cheng Chang, Yu~Cheng, Yiming Yang, and Barnabas Poczos.
\newblock MMD GAN: Towards deeper understanding of moment matching network.
\newblock In \emph{NeurIPS}, pages 2203--2213, 2017.

\bibitem[Micchelli and Pontil(2005)]{MicPon05}
Charles Micchelli and Massimiliano Pontil.
\newblock On learning vector-valued functions.
\newblock \emph{Neural Computation}, 17\penalty0 (1):\penalty0 177--204, 2005.

\bibitem[Neal(2001)]{Neal05}
Radford Neal.
\newblock Annealed Importance Sampling.
\newblock \emph{Statistics and computing}, 11\penalty0 (2):\penalty0 125--139, 2001.

\bibitem[Nemirovski et~al.(2009)Nemirovski, Juditsky, Lan, and
  Shapiro]{NemJudLanSha09}
Arkadi Nemirovski, Anatoli Juditsky, Guanghui Lan, and Alexander Shapiro.
\newblock Robust stochastic approximation approach to stochastic programming.
\newblock \emph{SIAM J. on Optimization}, 19\penalty0 (4):\penalty0 1574--1609,
  January 2009.
\newblock ISSN 1052-6234.

\bibitem[Nguyen et~al.(2008)Nguyen, Wainwright, and Jordan]{NguWaiJor08}
XuanLong Nguyen, Martin Wainwright, and Michael I. Jordan.
\newblock Estimating divergence functionals and the likelihood ratio by
  penalized convex risk minimization.
\newblock In \emph{NeurIPS}, pages 1089--1096, 2008.

\bibitem[Nowozin et~al.(2016)Nowozin, Cseke, and Tomioka]{NowCseTom16}
Sebastian Nowozin, Botond Cseke, and Ryota Tomioka.
\newblock f-GAN: Training generative neural samplers using variational
  divergence minimization.
\newblock In \emph{NeurIPS}, 2016.

\bibitem[Pistone and Rogantin(1999)]{PisRog99}
Giovanni Pistone and Maria~Piera Rogantin.
\newblock The exponential statistical manifold: mean parameters, orthogonality
  and space transformations.
\newblock \emph{Bernoulli}, 5\penalty0 (4):\penalty0 721--760, 1999.

\bibitem[Rahimi and Recht(2008)]{RahRec08}
Ali Rahimi and Ben Recht.
\newblock Random features for large-scale kernel machines.
\newblock In  \emph{NeurIPS}, 2008.

\bibitem[Rahimi and Recht(2009)]{RahRec09}
Ali Rahimi and Ben Recht.
\newblock Weighted sums of random kitchen sinks: Replacing minimization with
  randomization in learning.
\newblock In \emph{NeurIPS}, 2009.

\bibitem[Rezende et~al.(2014)Rezende, Mohamed, and Wierstra]{RezMohWie14}
Danilo~J Rezende, Shakir Mohamed, and Daan Wierstra.
\newblock Stochastic backpropagation and approximate inference in deep
  generative models.
\newblock In \emph{ICML}, pages 1278--1286, 2014.

\bibitem[Rezende and Mohamed(2015)]{RezMoh15}
Danilo~Jimenez Rezende and Shakir Mohamed.
\newblock Variational inference with normalizing flows.
\newblock In \emph{ICML}, 2015.

\bibitem[Rockafellar(1970)]{Rockafellar70}
R.~T. Rockafellar.
\newblock \emph{Convex Analysis}, volume~28 of \emph{Princeton Mathematics
  Series}.
\newblock Princeton University Press, Princeton, NJ, 1970.

\bibitem[Sion(1958)]{Sion58}
Maurice Sion.
\newblock On general minimax theorems.
\newblock \emph{Pacific Journal of mathematics}, 8\penalty0 (1):\penalty0
  171--176, 1958.

\bibitem[Sriperumbudur et~al.(2017)Sriperumbudur, Fukumizu, Gretton,
  Hyv{\"a}rinen, and Kumar]{SriFukGreHyvetal17}
Bharath Sriperumbudur, Kenji Fukumizu, Arthur Gretton, Aapo Hyv{\"a}rinen, and
  Revant Kumar.
\newblock Density estimation in infinite dimensional exponential families.
\newblock \emph{The Journal of Machine Learning Research}, 18\penalty0
  (1):\penalty0 1830--1888, 2017.

\bibitem[Strathmann et~al.(2015)Strathmann, Sejdinovic, Livingstone, Szabo, and
  Gretton]{StrSejLivSzaetal15}
Heiko Strathmann, Dino Sejdinovic, Samuel Livingstone, Zoltan Szabo, and Arthur
  Gretton.
\newblock Gradient-free hamiltonian monte carlo with efficient kernel
  exponential families.
\newblock In \emph{NeurIPS}, pages
  955--963, 2015.

\bibitem[Sugiyama et~al.(2010)Sugiyama, Takeuchi, Suzuki, Kanamori, Hachiya,
  and Okanohara]{SugTakSuzKanetal10}
Masashi Sugiyama, Ichiro Takeuchi, Taiji Suzuki, Takafumi Kanamori, Hirotaka
  Hachiya, and Daisuke Okanohara.
\newblock Conditional density estimation via least-squares density ratio
  estimation.
\newblock In \emph{AISTATS}, pages 781--788, 2010.

\bibitem[Sutherland et~al.(2017)Sutherland, Strathmann, Arbel, and
  Gretton]{SutStrArbArt17}
Dougal~J Sutherland, Heiko Strathmann, Michael Arbel, and Arthur Gretton.
\newblock Efficient and principled score estimation with nystr\" om
  kernel exponential families.
\newblock In \emph{AISTATS}, 2018.

\bibitem[Vembu et~al.(2009)Vembu, G{\"a}rtner, and Boley]{VemGarBol09}
Shankar Vembu, Thomas G{\"a}rtner, and Mario Boley.
\newblock Probabilistic structured predictors.
\newblock In \emph{UAI}, pages 557--564, 2009.

\end{thebibliography}
\end{document}